%% file: efficient_invariant_arxiv_new_version.tex
\documentclass{article}
\usepackage[utf8]{inputenc}

\usepackage{amsmath}
\usepackage{amssymb}
\usepackage{amsthm}
\usepackage{xcolor}
\usepackage{array} 
\usepackage{mathtools} 
\usepackage{wrapfig}
\usepackage[margin=1in]{geometry}
\usepackage{hyperref}

\newtheorem{theorem}{Theorem}[section]
\newtheorem{cor}[theorem]{Corollary}
\newtheorem{proposition}[theorem]{Proposition}

\theoremstyle{definition}
\newtheorem{example}[theorem]{Example}
\newtheorem{definition}[theorem]{Definition}
\newtheorem{lemma}[theorem]{Lemma}

\newtheorem{remark}[theorem]{Remark}

\newcommand{\B}{\mathcal{B}}
\newcommand{\CC}{\mathbb{C}}
\newcommand{\RR}{\mathbb{R}}
\newcommand{\NN}{\mathbb{N}}
\newcommand{\ZZ}{\mathbb{Z}}

\renewcommand{\det}{\mathrm{det}}

\newcommand{\M}{\mathcal{M}}
\newcommand{\N}{\mathcal{N}}
\newcommand{\Y}{\mathcal{Y}}

\newcommand{\Dint}{D_{\M}}

\newcommand{\rank}{\mathrm{rank}}
\newcommand{\sort}{\boldsymbol{\mathrm{sort}}}
\newcommand{\colsort}{\boldsymbol{\mathrm{colsort}}}

\newcommand{\finv}{f^{inv}}
\newcommand{\fgeneral}{f^{general}}

\newcommand{\bigO}{\mathcal{O}}
\newcommand{\Mweight}{\M_{weighted}}
\newcommand{\Mbinary}{\M_{binary}}

\newcommand{\nweight}{n_{weighted}}

\title{Low Dimensional Invariant Embeddings for Universal Geometric Learning}
\author{Nadav Dym\thanks{Faculty of Mathematics and Faculty of Computer Science, Technion, Haifa, Israel} \and Steven J. Gortler\thanks{School of Engineering and Applied Sciences, Harvard University, Cambridge, USA}}
\date{\today}

\begin{document}
	
	\maketitle

	\begin{abstract}
		
		This paper studies separating invariants: mappings on  $D$ dimensional domains which are invariant to an appropriate group action, and which separate orbits. The motivation for this study comes from the usefulness of separating invariants in proving universality of  equivariant neural network architectures.
		
		 We observe that in several cases the cardinality of separating invariants proposed in the machine learning literature is much larger than the dimension $D$. As a result, the theoretical universal constructions based on these separating invariants is unrealistically large. Our goal in this paper is to resolve this issue. 
		
		We show that  when a continuous family of semi-algebraic separating invariants is available, separation can be obtained by randomly selecting $2D+1 $ of these  invariants. We apply this methodology to obtain an efficient scheme for computing separating invariants for several classical group actions which have been studied in the invariant learning literature. Examples include matrix multiplication  actions  on point clouds by  permutations, rotations, and various other linear groups.
		
		Often the requirement of invariant separation is relaxed and only generic separation is required. In this case, we show that only $D+1$ invariants are required. More importantly, generic invariants are often significantly easier to compute, as we illustrate by discussing generic and full separation for weighted graphs. Finally we outline an approach for proving that separating invariants can be constructed also when the random parameters have finite precision.  
		
%
%
%
	\end{abstract}

    The goal of many machine learning tasks is to learn an unknown function which has some known symmetries. For example, consider the task of classifying 3D objects which are discretized as point clouds. In this scenario the unknown function maps a point cloud to a discrete set of labels. Depending on the type of 3D objects we are interested in, this function could be \emph{invariant} to translations, rotations, reflections, scaling and/or affine transformations. It is also typically invariant to permutation- reordering the points does not impact the underlying object. 
    
    Another example comes from physics simulations, where we aim to learn a function which is given a point cloud at time $t=0$, which represents the positions of a collection of physical entities such as  particles or planets, and maps them to their positions at time $t=1$. This map is \emph{equivariant} with respect to translations, orthogonal transformations (and Lorenz transformations in the relativistic setting) and permutations: applying these transformations to the input will lead to a corresponding transformation of the output.  
    
    In machine learning approaches, an approximation for the unknown function is sought for, from within a parametric hypothesis class of functions. Equivariant machine learning considers hypothesis classes which by construction have some or all of the symmetries of the unknown functions. For example, there are multiple neural network architectures for point clouds  which are equivariant (or invariant)
    to permutations \cite{wang2019dynamic,zaheer2017deep,qi2017pointnet} and/or geometric transformations such as rotations \cite{kondor2018clebsch,thomas2018tensor,deng2021vector}, orthogonal transformations \cite{batzner2021e,satorras2021n,yao2021simple}, or Lorenz transformations \cite{bogatskiy2020lorentz,fan2022nested,hao2022lorentz}. 
    
    Perhaps the most famous result in the theoretical study of general neural networks is the fact that fully connected ReLU networks can approximate any continuous function \cite{pinkus1999approximation}. Analogously, it is desirable to devise  invariant (or equivariant) neural networks which are universal in the sense that they can approximate \emph{all} continuous invariant (or equivariant) functions. This question has been studied in many works such as \cite{finkelshtein2022simple,yarotsky2022universal,puny2021frame,bokman2021zz,dym2020universality,keriven2019universal,sannai2019universal,maron2019universality,segol2019universal}.    
   
    Universality results for invariant networks typically consider functions which can be written as
        \begin{equation}\label{eq:char}
     f=\fgeneral\circ \finv, \text{ where } \finv:V\to \RR^m, \fgeneral:\RR^m\to \RR .
     \end{equation}
 where $V$ is a Euclidean space, $G$ is a group acting on it, $\finv$ is a $G$-invariant continuous function, $\fgeneral$ is continuous, and so $f$ is $G$-invariant and continuous. If $\fgeneral$ comes from a function space rich enough to approximate all continuous functions (typically a fully connected neural network), and  $\finv$ is a continuous invariant mapping which separates orbits- which means that $\finv(x)=\finv(y)$ if and only if $x=gy$ for some $g \in G$, then invariant universality is guaranteed (see e.g,, Proposition~\ref{prop:uniOrbit}). As discussed in \cite{villar2021scalars}, orbit separating invariants can also be used to construct universal equivariant networks. Thus the question of invariant and equivariant universality is to a large extent reduced to the related question of invariant separation, which is the focus of this paper.

We note that invariant mappings on $V$ induce well defined mappings from $V/G$ to some $\RR^m$, and 
separation of invariant mappings is equivalent to injectivity 
of this map. In practice, our mappings will all be continuous, and often, their inverses will also be continuous
(though we do not explore this issue in the current paper).
For this reason we will  informally refer to separating invariant mappings as invariant \emph{embeddings}. 
 
 In the equivariant learning literature, orbit separating invariant functions are typically supplied by classical invariant theory results, which in many cases can characterize all polynomial invariants of a given group action through a finite set of invariant polynomial generators. However, the number of generators is often unrealistically large.  For example, the classic point net architecture \cite{qi2017pointnet} 
 is a permutation invariant neural network which operates on point clouds in $\RR^{d\times n}$. Proofs of the universality of this network are typically based on its ability to approximate power sum polynomials, which are generators of the ring of permutation invariant polynomials (see e.g., \cite{segol2019universal,zaheer2017deep,maron2019provably}). However, 
 while in  the experimental setup reported in \cite{qi2017pointnet}, $n=1024,d=3$ and the network creates $1024$ invariant features, the number of power sum polynomials is 
 ${n+d}\choose{n} $, 
 which amounts to over $180$ million invariant features in this case.
 
 There are various mathematical results from different disciplines which suggest that the number of separating invariants need not be larger than roughly twice the dimension of the domain (in particular, in the example discussed above 
 this would mean that only $2\dim(\RR^{d\times n})=2n\cdot d=2048$ permutation invariant features would be needed, rather than $180$ million). For example, in invariant theory it is known that (for groups acting on affine varieties over algebraically closed fields by automorphisms), while the number of generators of the invariant ring is difficult to control, there are separating sets whose size is $2D+1$, where $D$ 
 is the Krull dimension of the invariant ring \cite{derksen2015computational,dufresne2008separating}. Additionally, continuous orbit separating mappings on $V$ can be identified with continuous \emph{injective} mappings on the quotient space $V/G$. Thus if $V/G$ is a smooth manifold  Whitney's embedding theorem shows that it can be easily be (smoothly) embedded into $\RR^{2D+1}$, where $D$ is the dimension of $V/G$. (With more work, this can be reduced to 
 $2D$.)
 A similar theorem holds under the weaker assumption that $V/G$ is a compact metric space with topological dimension $D$ (\cite{Munkres}, page 309).  Additionally, it is a common assumption (e.g., \cite{ansuini2019intrinsic,shaham2018provable,law2006incremental} ) that the data of interest in a machine learning task typically resides in a `data manifold' $\M\subseteq V$ whose dimension $\Dint$ (the `intrinsic dimension') is significantly smaller than the dimension of $V$ (the `ambient dimension'). This assumption is at the root of many dimensionality reduction techniques such as auto-encoders, random linear projections, PCA, etc. In this scenario, we expect to achieve orbit separation with only $\approx 2\Dint$ invariants.

\begin{figure}[t]
	\includegraphics[width=\columnwidth]{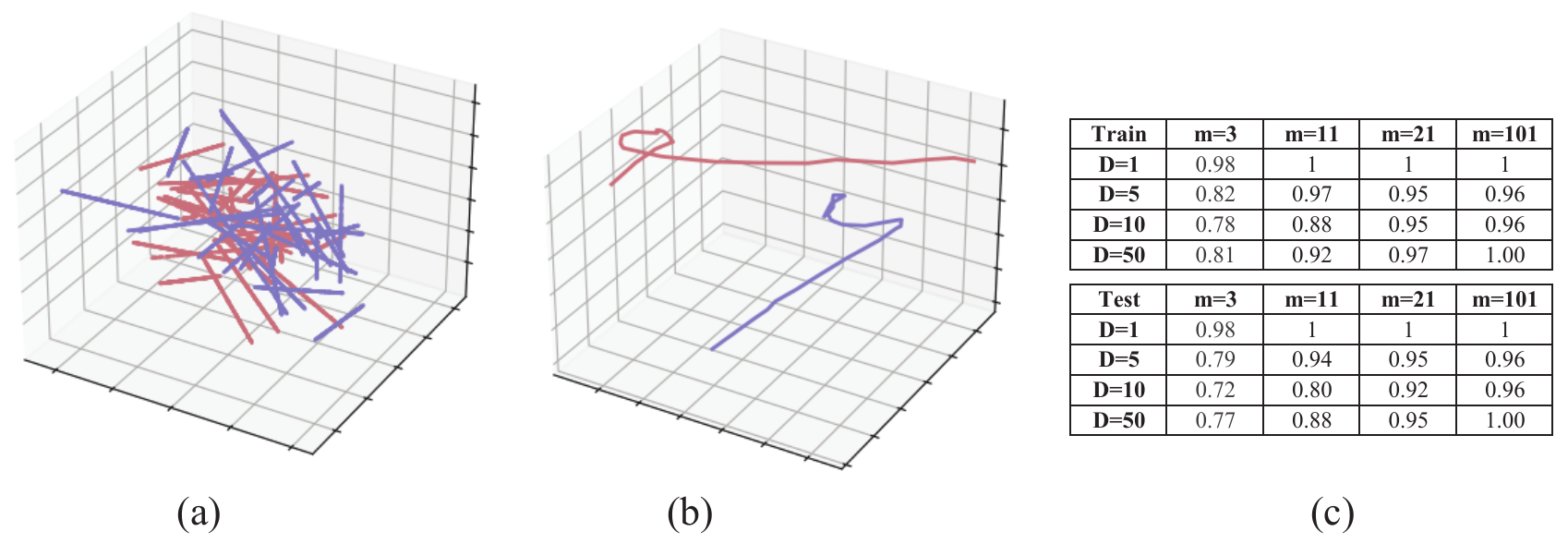}
	\caption{\small Figure~\ref{fig:sort} (a) shows  two lines in $\RR^{3\times n}$ and their image under (some) random $S_{n}$ permutations, visualized by projecting into $\RR^3$
 Subplot (b) shows the image of these lines (dimension $D=1$) under the mapping we describe in \eqref{eq:balan} with $m=2D+1=3 $. As guaranteed by Proposition~\ref{prop:Sn}, these images do not intersect. The  tables in Subplot (c) shows the training and test error when training an MLP to classify the curves based on the embedding in (b), for various values of $D$ and $m$.}
	\label{fig:sort}		
\end{figure}
 Based on the discussion above, we formulate the first goal of this paper:

 \textbf{First Goal:} Provide an algorithm for computing $\approx 2\Dint$ separating invariants for group actions on a `data manifold' of intrinsic dimension $\Dint$.

Figure \ref{fig:sort} illustrates our first goal and our discussion so far (full details of the experiment are given in Section \ref{sec:experiments}). We take two lines in high dimensional space, $\RR^{3\times n}$,
(colored orange and blue) and apply many random permutations 
from $S_n$
to the points on both lines thus obtaining many orange and blue lines. (In (a) we visualize this data projected down to $\RR^3$.) In (b) we see the image of these lines under a permutation invariant separating mapping to $\RR^3$. Due to invariance all orange (respectively blue) lines are mapped to a single orange (or blue) curve in $\RR^3$. Due to separation these curves do not intersect.  It is possible to achieve separation with only three coordinates since the intrinsic dimension of the lines is $\Dint=1$.  

The first row in the tables in Figure \ref{fig:sort}(c) shows that applying a standard neural networks architecture to the invariant curves in (b) can accurately classify points according to the line from which they originated. The other rows show that when the intrinsic dimension $\Dint$ of the data is increased, more invariant features are needed, where $2\Dint+1$ features typically give accurate results.   

\textbf{Dimensionality reduction by linear projections} In Corollary\ref{cor:linear} we provide a very general tool to achieve our first goal: with some minor assumptions, we show that if a finite set of $N$ separators is known, then the number of separators can be reduced to $2\Dint+1$ by taking $2\Dint+1$  random linear combinations of the original separators.

While we are not aware of a result as general as Corollary\ref{cor:linear} in terms of the domains and groups which it can handle, we note that the idea of dimension reduction of separators using random linear projections is not new. It is used in many of the proofs mentioned above which guarantee $\approx 2\Dint $ separators in the algebraic setting (see e.g., \cite{dufresne2008separating}), and was used for specific group actions in \cite{balan2022permutation} and \cite{cahill2020complete}. Random projections are also used
in the Whitney embedding theorem.

 A significant computational drawback of the linear projection technique is that it still necessary to compute the original large set of separators before the linear projection. Thus the total complexity of computing the separators is not improved by taking linear projections. Accordingly, we formulate our second goal, which is in fact just a refinement of the original goal

 \textbf{Second goal (refinement of first goal):}
 Provide an \emph{efficient, polynomial time} algorithm for computing $\approx 2\Dint$ separating invariants for group actions on a `data manifold' of intrinsic dimension $\Dint$.
 In our setting we do this by starting with some continuous family of maps such that this entire family is separating, and
 show that we can separate with only $\approx 2\Dint $ 
 randomly chosen elements from this family.

\subsection{Main results}\label{sub:main_results} Our main result in this paper is a methodology for addressing the second goal and \emph{efficiently} computing  $2\Dint+1$ separating invariants for several classical group actions which have been studied in the invariant machine learning community. Our methodology is inspired by the results in \cite{balan2006signal}, which uses tools from real algebraic geometry to show orbit separation in the context of the phase retrieval problem. Our results, formally stated in Theorem~\ref{thm:useful}, can be seen as a generalization of these results to general groups. To illustrate the theorem and its usefulness we will give an example which will later be discussed in more detail in Subsection \ref{sub:SL}.

 \begin{example}\label{ex:SLd}
Let $\M$ denote the collection of all $d\times n$ matrices which have rank $d$ (assume that $d\leq n$). This is a set of dimension $\Dint=n\cdot d$. Consider the action of $SL(d)=\{A\in \RR^{d\times d}| \quad \det(A)=1 \} $ on $\M$ by multiplication from the left. A natural way to construct invariants for this group action is to pick a subset $I\subseteq \{1,\ldots,n\} $ of $d$ indices and considering the function $X\mapsto \det(X_I) $, where $X_I$ is the $d\times d$ matrix obtained by choosing the $d$ columns of $X$ indexed by $I$. In fact, it is known that the functions $\det(X_I)$ are generators of the invariant ring, and as a result are separators. The trouble is that  the number of subsets of size $d$, and so the number of generators, is a prohibitive $n \choose d $. 
Using Corollary\ref{cor:linear}, we can reduce the number of separators by choosing $w^{(1)},\ldots,w^{(2nd+1)} $ random vectors in $\RR^{n\choose d} $ and considering functions of the form 
\begin{equation} \label{eq:expensive}
    X\mapsto \sum_{I\subseteq \{1,\ldots,n\}, \quad |I|=d}  w_I^{(j)} det(X_I), \quad j=1,\ldots,2nd+1.
\end{equation}
However computing each one of these invariant still has complexity of $\sim {n \choose d}$. 

Out methodology to improve upon this issue uses a  family of invariants  parameterized by the continuous `weight matrix' $W$: we note that for every $W\in \RR^{n\times d} $ the function 
$$X\mapsto \det(XW) $$
is $SL(d)$ invariant. Additionally, we note that every generator $det(X_I)$ is of the form $det(X_I)=det(XW_I) $ for an appropriate choice of a $n\times d $ matrix $W_I$. In particular this means that the value obtained by $\det(XW)$ for all possible $W$ determines $X$ uniquely, up to $SL(d)$ equivalence. Theorem~\ref{thm:useful} shows that in this scenario, if we 
simply choose random $W^{(1)},\ldots,W^{(2nd+1)} $ then the functions
$$X\mapsto \det(XW^{(j)}), \quad j=1,\ldots,2nd+1 $$
will be invariant and separating. In contrast to the separating invariants in \eqref{eq:expensive}, the complexity of computing each one of these invariants is polynomial in $n,d$.
 \end{example}

Informally, Theorem~\ref{thm:useful} can be stated as follows: 
Suppose we can find a  family of invariants $p(\cdot,w)$ parameterized by the continuous `weight vector' $w$, such that every orbit pair is
separated by some $w$.
Then  for almost every random selection of vectors $w_i$ with $i=1,\ldots,2\Dint+1$, the invariants $x\mapsto p(x,w_i)$ will separate orbits. In a sense, the usefulness of Theorem~\ref{thm:useful} is that it reduces the problem of separating \emph{all} points with a \emph{finite} number of invariants to the problem of separating \emph{pairs} of points with a \emph{continuous} family of invariants. One could draw a vague analogy here to the usefulness of the Stone-Weierstrass theorem in reducing approximation proofs to separation of pairs of points.

 Roughly speaking, the assumptions needed for this theorem are (i) that $\M$ is a semi-algebraic set- that is, a set defined by polynomial equality and inequality constraints, and (ii) that the function $p(x,w)$ is a semi-algebraic function. This is a rather large class of functions which includes polynomials, rational functions and continuous piece-wise linear functions. 

The usefulness of Theorem~\ref{thm:useful} for obtaining efficient separating invariants was illustrated in Example \ref{ex:SLd}.   In Section~\ref{sec:applications} we show similar applications for several other group actions: we  find $2\Dint+1$ orbit separating invariants for the action of permutations on point clouds which we discussed above. These invariants are continuous piece-wise linear functions obtained as a composition of the `sort' function with a linear function from the left and right.  The complexity of computing each invariant is (up to a logarithmic factor) linear in the ambient dimension. Even when we are interested in separation in all of $\M=\RR^{d\times n} $, so that $\Dint=n\cdot d$, the number of invariants we need for separation is significantly lower than the ${n+d}\choose{n} $ power sum polynomials discussed above. This advantage becomes more pronounced when the `data manifold' $\M$ is low dimensional. 

Similarly, we construct $2\Dint+1$ separating invariants for the action  orthogonal transformations $O(d)$ and special orthogonal transformations $SO(d) $, on $d$ by $n$ point clouds, which can be compared with the standard separating invariants used for these group actions which have cardinality of $n \choose 2 $ and $\sim{n \choose d} $ respectively. We also construct $2\Dint+1 $ separating invariants for Lorenz transformations (and other isometry groups),  the general linear group, and for scaling and translation. For several of the latter results we need to assume that the `data manifold' $\M$ contains only full rank point clouds. A summary of these results, and the complexity of computing each invariant, is given in Table~\ref{table}.

\textbf{Generic orbit separation} 
As suggested in several works (e.g., \cite{puny2021frame,villar2021scalars,rong2021almost}) the notion of orbit separating invariants can be replaced with a weaker notion of generically orbit separating invariants- that is, invariants defined on $\M$ which are separating outside a subset $\N$ of strictly smaller dimension. This weaker notion of separation is sufficient to prove universality only for compact sets in $\M \setminus \N$. Another disadvantage is that it is not inherited by subsets: generic orbits separation on $\M$ is not necessarily preserved on a subset $\M' \subseteq \M$, since it is even possible that $\M'$ is completely contained in $\N$. The advantage of generic separation is that it is generally easier to achieve. Indeed, in Theorem~\ref{thm:useful} we show that while we need  $2\Dint+1$ random measurements for  orbit separation, only $\Dint+1$ measurements are sufficient for generic separation. This results resembles classical results which show  that for an irreducible algebraic
variety of dimension $D$ embedded in high dimension,  almost every projective  projection down to dimension $D+1$
will be generically one-to-one. (For example see Example 
7.15 and Exercise 11.23 in~\cite{harris}.) 

A more significant computational advantage in settling for generic invariants is that each invariant can usually be computed more efficiently. We exemplify this by considering graph valued permutation invariant functions: there is no known algorithm to separate graphs (up to permutation equivalence) in polynomial time \cite{grohe2020graph}. Correspondingly, while we can easily use our methodology to find a small number of separating invariants, the computational price of computing these invariant is prohibitive. On the other hand, it is known that for `generic graphs' \cite{dym2018exact,aflalo2015convex} separating graphs is not difficult, and correspondingly we are able to find a small number of generically separating invariants which can be computed efficiently. 


\begin{table}[t]
	\input{group_table}
	\caption{Summary of the results in Section~\ref{sec:applications}. For each group action we describe a parametric family of separators, the complexity of computing each separator, the domain on which separation is guaranteed, and the number of generators for this group action. In comparison, the number of random separators needed in all examples is $2nd+1$ (or  $2\Dint+1$ when considering separation on a semi-algebraic subset $\M $ )}
	\label{table}
\end{table}

\subsection{Related work}\label{sub:related}
\paragraph{Phase retrieval and generalizations}
As mentioned above, our results were inspired by orbit separation results in the phase retrieval literature. Phase retrieval  is the problem of reconstructing a signal $x$ in $\CC^n$ (or $\RR^n$), up to a global phase factor, from magnitudes of linear measurements without phase information $|\langle x,w_i \rangle |, i=1,\ldots,m$, where $w_i$ are complex (or real) $n$ dimensional vectors. In our terminology, the goal is to find when these measurements are separating invariants with respect to the action of the group of unimodular complex number $S^1$ (or real unimodular numbers $\{-1,1\}$).  

In \cite{balan2006signal} it is shown that $m=2n-1$ random linear measurements in the real setting, or $m=4n-2$  random linear measurements in the complex setting, are sufficient to define a unique reconstruction of the signal up to global phase. In \cite{conca2015algebraic}, the number of measurements needed for attaining separation in the complex setting is slightly reduced to  $m=4n-4$. 

In \cite{evans2020conjugate}, conjugate phase retrieval is discussed: here the vectors $w_i$ are real measurements of complex signals in $\CC^n$. These measurements are invariant with respect to global phase multiplication and also complex conjugation, and it is shown that $4n-6$ generic measurements are separating with respect to this group action.   

The separation results obtained for \emph{real} and \emph{conjugate} phase retrieval, are equivalent to saying that on the space of $d\times n$ real matrices $X$, around $\sim 2dn$ generic measurements of the form $\|Xw_i\|$ are separating with respect to the action of $O(d)$, for the cases $d=1$ (real phase retrieval) and $d=2$ (conjugate phase retrieval).  In Subsection~\ref{sub:Od} we use our methodology to show that for $d\geq 2$, separation with respect to the action of $O(d)$ can be obtained by choosing $2nd+1 $ random measurements of the form $\|Xw_i\|^2 \quad i=1,\ldots,2nd+1$. We note that this result is in essence not new, as it can be derived easily from results on the recovery of rank one matrices from rank one linear measurements (see \cite{rong2021almost}, Theorem 4.9 )

There are two natural ways to generalize the separation results of \emph{complex} phase retrieval: the first is to consider  measurements $\|Xw_i\|$ with $w_i\in \CC^d$  and $X\in \CC^{d\times n}$. These measurements are invariant with respect to the action of $d\times d$ unitary matrices on $\CC^{d\times n}$ by multiplication from the left. In \cite{kech2017constrained} it is shown that $4d(n-d)$ such generic measurements are separating with respect to the unitary action. For $d=1$ this coincides with the $4n-4$ invariants needed for complex phase retrieval.   

An alternative generalization of complex phase retieval is searching for separating invariants for the action of $SO(d)$ on $\RR^{d\times n} $. Complex phase retrieval gives us separating invariants for the special case $d=2$, using the identifications $S^1 \cong SO(2) $ and $\CC \cong \RR^2 $.  In  Subsection~\ref{sub:SOd} we note that when $d>2$ separation is not possible with measurements which involve only norms and linear operations, but can be achieved by adding an additional `determinant term' to the norm based measurements.

While our results on invariant separation for  $SO(d)$ are new (to the best of our knowledge), we feel that the main novelty in this work is in the generalization of the basic real algebraic geometry arguments used in the phase retrieval proofs to a very general class of real group actions (Theorem~\ref{thm:useful}) and invariant functions (semi-algebraic mappings), and in the ability to apply this same methodology to multiple group actions (see Table~\ref{table}).

Finally, we note that the  results quoted above show that  it is often possible to get generic separation even when the number of invariants is slightly smaller than the $2\Dint+1$ promised by our Theorem~\ref{thm:useful}. We do not pursue the optimal cardinality for separation in this paper for two reason. First, this typically requires a case-by-case analysis and in this paper we try to focus on highlighting a general principle that can be applied to a wide number of settings. Secondly, ultimately the number of separating invariants is not likely to be smaller than the dimension of $\M/G$, and in most applications we are interested in $\dim(\M/G)\approx \dim(\M) $, so that ultimately we do not expect an improvement in separation cardinality by more than a factor of $2$.


\paragraph{Permutation Invariant Machine Learning}
Obtaining $\bigO(n\cdot d)$ separating invariants for the action of the permutation group $S_n$  on $\RR^{d\times n}$ is trivial when $d=1$, as discussed in e.g., \cite{wagstaff2019limitations}. When $d>1 $,   a separating set of invariants is suggested in \cite{balan2022permutation} which is a composition of row-wise sorting with linear transformations from the left and the right. They show that this construction is separating when the intermediate dimension is very high (larger than $n!$), and that whenever separation is achieved, this map is also Bi-Lipschitz. We will show that in fact the intermediate dimension need only be $2nd+1$ (or lower when the `data manifold' has lower dimension), and that the second, larger matrix in their construction can have a certain sparse structure which was previously used in \cite{zhang2019fspool} (see Remark~\ref{remark:balan}). We note that when $d=2$ the sufficiency of $\sim n$ and even $\sim \log(n)$ measurements was shown in \cite{matouvsek2008many,renyi1952projections}

\paragraph{Recent results} Shortly after the first preprint of this manuscript appeared, the authors of \cite{cahill2022group,mixon2022injectivity,mixon2022max} proposed invariants called max filters which are defined for  all Hilbert spaces with isometric  group actions. For finite dimensional Hilbert spaces these invariants were shown to be separating using our Theorem~\ref{thm:useful}. Max filters can thus be suggested as alternative separating invariant to those we suggest in Section~\ref{sec:applications} for the actions of $S_n$, $O(d) $ and $SO(d)$, which are all groups of linear isometries. An attractive attribute of this approach is that the same invariant fits all cases, while an advantage of the invariants we present here is that they are more efficient to compute (computing a max filter in these examples is done by solving an optimization problem with relatively high complexity). 

Another recent application of our work here, was a derivation of (relatively) efficient separating invariants for the joint action of $SO(d)\times S_n $ or $O(d)\times S_n $ on $\RR^{d\times n}$. This is described in \cite{hordan2023complete}.

	\textbf{Paper organization} 
	 The structure of the paper is as follows: 
	In Section~\ref{sec:main} we provide some mathematical background, and use it to state and prove our main theorem. We then discuss in general terms the ways in which this theorem can be used to devise separating invariants of low complexity for various group actions, and relationships to concepts from invariant theory such as generating invariants and polarization..
	
	In Section~\ref{sec:applications} we describe several applications of the theorem, showing in several examples of interest how a low-dimensional set of separating invariants, which can be computed efficiently, can be obtained using our methodology. These examples include point clouds with multiplication by permutation matrices from the right, or multiplication by orthogonal transformations, rotations, volume preserving linear transformations,  general linear transformations, from the left. We also discuss the more trivial scaling and translation actions.

	In Section~\ref{sec:generic} we discuss generic separation. We show that generic separation can be obtained using only $\Dint+1$ invariants, and show that generic separation can be computed efficiently for weighted graphs while full separation is unlikely due to the fact that the graph isomorphism problem has no known polynomial time algorithm. 
	
	In Section~\ref{sec:computer_numbers} we give an outline of an argument that shows that separation can be obtained also if the parameters of the separating invariants we consider have finite precision. This argument is applicable only for some of the polynomial invariants we consider here. Finally  in Section~\ref{sec:experiments} we provide some initial experiments, showing a simple permutation invariant classification problem on point clouds in $\RR^{d\times n}$ with high ambient dimension $n\times d$, and low intrinsic dimension $\Dint$,  can be efficiently solved using  $2\Dint+1 $ of our separating invariants, as our theory predicts.

	\section{Definitions and main theorem}\label{sec:main}
	\subsection{Notation}
	We denote matrices  $X$ in $\RR^{d\times n}$ by capital letters and refer to them as point clouds. The  columns of a point cloud are denoted by little letters $X=(x_1,x_2,\ldots,x_n)$. We use $1_n$ to denote the constant vector $(1,1,\ldots,1)\in \RR^n$, and $e_i\in \RR^n $ to denote the $n$-dimensional vector with $1$ in the $i$-th coordinated and zero in the remaining coordinates. 
	\subsection{Mathematical background}
	We begin this section by explaining how continuous separating invariant functions are used to characterize all continuous invariant functions.  We then lay out some definitions we need for our discussion and prove our main theorem regarding the construction of low-dimensional continuous separating invariant functions (Theorem~\ref{thm:useful}). 
	\paragraph{Universality and orbit separation}
	We begin with defining invariant functions and orbit separation
\begin{definition}
	Let $G$ be a group acting on a set $\M$, let $\Y$ be a set, and let $f:\M\to \Y$. We say that $f$ is \emph{invariant} if $f(x)=f(gx)$ for all $g\in G, x\in \M$. We say that a subset $\M'$ of $\M$ is stable under the action of $G$  if  $gm\in \M'$ for all $g\in G$ and $m\in \M'$.
\end{definition}
\begin{definition}
	Let $G$ be a group acting on a set $\M$, let $\Y$ be a set, and let $f:\M \to \Y$ be an invariant function.We say that $f$ \emph{separates orbits} if $f(x)=f(y)$ implies that $x=gy$ for some $g\in G$. We say that a finite collection of invariant functions $f_i:\M\to \Y, i=1,\ldots,N$ separates orbits, if the concatenation $(f_1(x),\ldots,f_N(x))$ separates orbits.  
\end{definition}
The following proposition, proved in the appendix, shows that every continuous invariant function can be written as a composition of a orbit separating, continuous invariant functions and a continuous (non-invariant) function.
	\begin{proposition}\label{prop:uniOrbit}
	Let $\M$ be a topological space, and $G$ a group which acts on $\M$. Let $K\subset \M$ be a compact set, and let $\finv:\M\to \RR^N$ be a continuous $G$-invariant map that separates orbits. Then for every continuous invariant function  $f:\M \to \RR$ there exists some continuous $\fgeneral:\RR^N \to \RR$ such that 
	$$f(x)=\fgeneral(\finv(x)), \quad \forall x\in K $$
	\end{proposition}
Somewhat more complicated characterizations of \emph{equivariant} functions via separating invariants are described in \cite{villar2021scalars}.

Now that we have established the importance of continuous separating invariant for approximation of continuous invariant (or equivariant) functions, we will exclusively focus on the topic of finding a small collection of such continuous separating invariant functions. Our technique for doing so relies on several concepts from real algebraic geometry which we will now introduce.   
\paragraph{Real algebraic geometry}
Unless stated otherwise, the background on real algebraic geometry presented here is from \cite{Basu}. 
\begin{definition}[Semi-algebraic sets]
	A real semi-algebraic set in $\RR^k$ is a finite union of sets of the form
	$$\{x\in \RR^k| p_i(x)=0 \text{ and } q_j(x)>0 \text{ for } i=1,\ldots,N \text{ and } j=1,\ldots,m\} $$
where $p_i$ and $q_j$ are multivariate polynomials with real coefficients.
\end{definition}
Semi-algebraic sets are closed under finite unions, finite intersections, and complements. We next define semi-algebraic functions
\begin{definition}[Semi-algebraic functions]
	Let $S\subseteq \RR^\ell$ and $T\subseteq \RR^k$ be semi-algebraic sets. A function $f:S\to T$ is semi-algebraic if 
	$$Graph(f)=\{(s,t)\in S \times T| \quad  t=f(s) \} $$ 	
	is a semi-algebraic set in $\RR^{\ell+k}$.
\end{definition}
Polynomials $f:\RR^\ell \to \RR^k $ are obviously semi-algebraic functions. Similarly given two polynomials $p_1,p_2:\RR^{\ell} \to \RR$ the rational function $q(x)=\frac{p_1(x)}{p_2(x)}$ is well-defined and semi-algebraic on the semi-algebraic set $\{x\in \RR^\ell| \, p_2(x)\neq 0\}$.  

In addition, assume we are given a collection of semi-algebraic sets $S_1,\ldots,S_n\subseteq \RR^\ell $ whose union is all of $\RR^\ell$, and a function $f$ whose restriction to each $S_i$ is a polynomial $f_i$. We call such functions piece-wise polynomial functions. Piece-wise polynomials are semi-algebraic functions since 
$$Graph(f)=\cup_{i=1}^n \{(s,t)| \, s\in S_i \text{ and } t=f_i(s)\}$$
In particular, this class of functions include ReLU neural networks and the sorting function we will use in Subsection~\ref{sub:Sn}, which are continuous piece-wise linear functions. Piece-wise linear functions are  a special case of piece-wise polynomial functions where each semi-algebraic set $S_i$ is a closed convex polyhedron, and each $f_i$ is an affine functions.

\begin{wrapfigure}{R}{0.5\textwidth}
	\includegraphics[width=0.45\columnwidth]{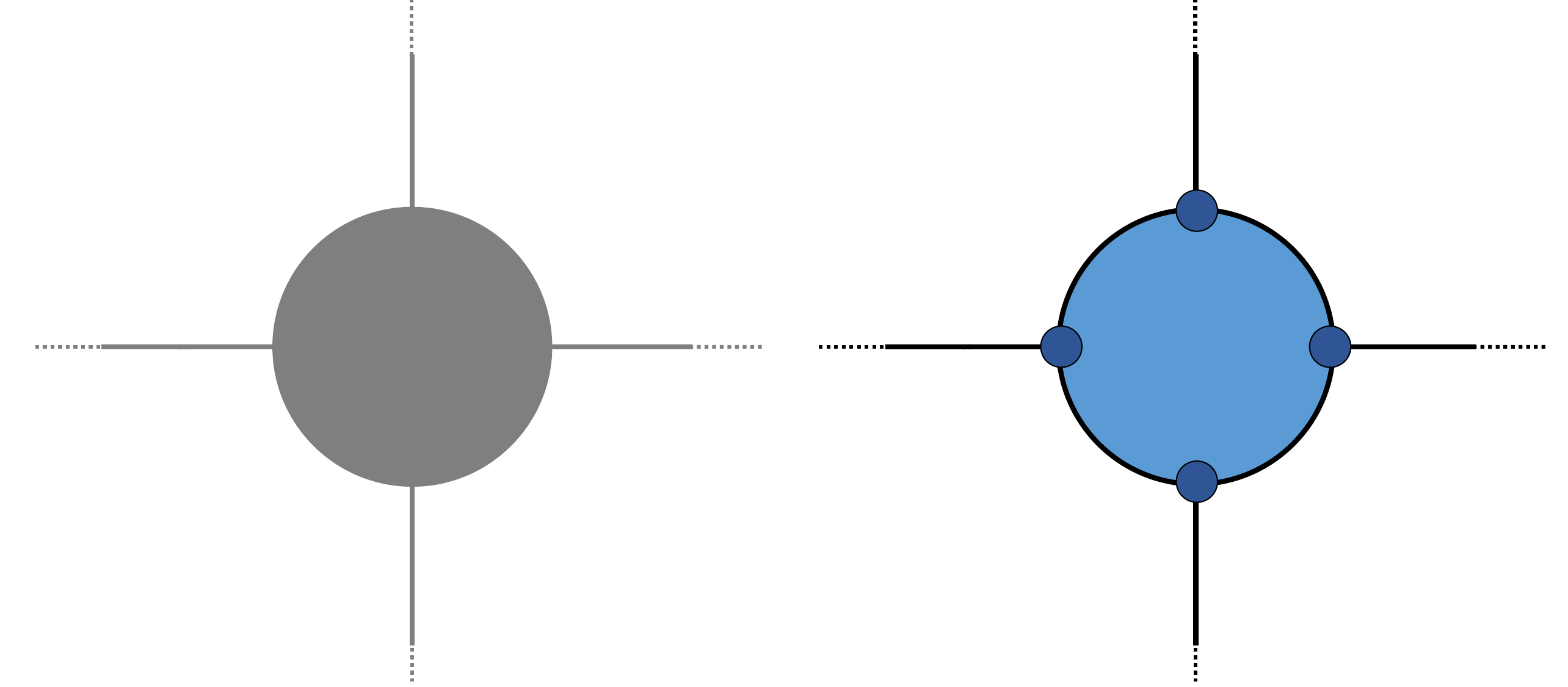}
	\caption{\small Stratification of a semi-algebraic set. See explanation in main text.}
	\label{fig:burner}		
\end{wrapfigure}
\paragraph{Stratification and dimension} A semi-algebraic set $S$ can be written as a finite  union of pairwise disjoint sets $S_1,\ldots,S_n$ such that each $S_i$ is a $C^{\infty}$ manifold of dimension $r_i$, and the closure of each $S_i$ in $S$ contains only $S_i$ itself, and sets $S_j$ with $r_j<r_i $. This decomposition is called a \emph{stratification}  (see  \cite{Basu}, page 177). The dimension of $S$ is the maximal dimension $\max_{1\leq i \leq n}r_i $ of all the manifolds in the decomposition (this definition of dimension can be shown to be independent of the stratification chosen). 

Figure~\ref{fig:burner} shows a stratification of the semi-algebraic set 
	$$S=\{(x,y)| \quad 1-x^2-y^2>0\}\cup \{(x,y)| \quad 1-x^2-y^2=0\} \cup \{(x,y)| \quad xy=0\} .$$
	The set is shown in the left of the figure, and the stratification is visualized on the right. It includes a single two-dimensional open disc, eight open curves (dimension 1), and four points (dimension 0). Hence the dimension of $S$ is two.

\paragraph{Families of invariant separators}
We now introduce some definitions needed for discussion of group actions and separation using a real algebraic geometry framework.




Assume that $G$ is a group acting on a set $\M$. The \emph{orbit} of $x\in \M$ under the action of a group $G$ is the set 
$$[x]=\{y\in \M| \, \exists g\in G \text{ such that } y=gx \} $$  
When $y$ is in the orbit of $x$ we use the notation $x \sim_Gy$, and when it is not in the orbit of $x$ we use the notation $x \not \sim_Gy$.

	\begin{definition}
		Let $G$ be a group acting on a semi-algebraic set $\M$, and $D_w$ be an integer
		greater than or equal to one.
		We say that a semi-algebraic function
		$$p:\M \times \RR^{D_w}\to \RR $$
		is \emph{a family of $G$-invariant semi-algebraic functions}, if for every $w\in \RR^{D_w} $ the function $p(\cdot,w)$ is $G$ invariant.
		
		We say a family of $G$-invariant semi-algebraic functions \emph{separates} orbits in $\M$, if for all $x,y\in \M$ such that  $x \not \sim_G y$ there exists $w\in \RR^{D_w} $ such that  $p(x,w)\neq p(y,w) $.
		
		We say a family of $G$-invariant semi-algebraic functions \emph{strongly separates} orbits in $\M$, if for all $x,y \in \M$ with  $x \not \sim_G y$, the set
		$$  \{w\in \RR^{D_w}| \, p(x,w)= p(y,w)\} $$ 
		has dimension $\leq D_w-1$. 	 
	\end{definition}
We note that if  $p(x,w)$ is polynomial in $w$ for every fixed $x$, then separation implies strong separation, since the set of zeros of a polynomial which is not identically zero is always dimensionally deficient.

\subsection{Main Theorem}
We now have all we need to state our main theorem:
	\begin{theorem}\label{thm:useful}
		Let $G$ be a group acting on a  semi-algebraic set $\M$ of dimension $dim(\M)=\Dint$. Let $p:\M\times \RR^{D_w} \to \RR$ be a  family of $G$-invariant semi-algebraic functions. 
		 If $p$ strongly separates orbits in $\M$, then for Lebesgue almost-every $w_1,\ldots,w_{2\Dint+1}\in \RR^{D_w} $ the $2\Dint+1$ $G$-invariant semi-algebraic functions
			$$p(\cdot,w_i), i=1,\ldots,2\Dint+1 $$
			separate orbits in $\M$.
		
	\end{theorem}


The remainder of this subsection is devoted to proving this theorem. At a first reading we recommend skipping to Subsection~\ref{sub:method} at this point. 

We begin by recalling some additional real algebraic geometry facts we will need for the proof, also taken from \cite{Basu}. We first recall some  basic properties of real algebraic dimension:
\begin{enumerate}
	\item If $A\subseteq B \subseteq \RR^\ell$ are semi-algebraic sets then 
	$$\dim(A)\leq \dim(B) $$
	\item If $A \subseteq \RR^\ell$ and $B \subseteq \RR^m$ are semi-algebraic sets then 
	$$\dim(A\times B)=\dim(A)+\dim(B) $$
	\item If $S\subseteq \RR^k$ is a semi-algebraic set and  $f:S\to \RR^\ell$ is a semi-algebraic function then  $f(S)$ is a semi-algebraic set and
	$$\dim(f(S))\leq \dim(S) $$
	If $f$ is a diffeomorphism then we have equality $\dim(f(S))=\dim(S)$.
	\item If $A \subseteq \RR^\ell$ is a semi-algebraic set of dimension strictly smaller than $\ell$ then it has Lebesgue measure zero.
\end{enumerate}
Another useful fact we will use is that the projection of a semi-algebraic set is also a semi-algebraic set.

We next state and prove the following lemma
\begin{lemma}\label{lem:dim}
	Let $S\subseteq\RR^{D_1}$ be a semi-algebraic set and $f:\RR^{D_1}\to \RR^{D_2}$ a polynomial. Assume that for all $t\in f(S)$ we have that $\dim(f^{-1}(t))\leq \Delta$, then 
	$$\dim(S)\leq \dim(f(S))+\Delta $$ 	
\end{lemma}
\begin{proof}
	Denote $\Delta_S=\dim(S)$. Let $S_i, i=1,\ldots,N$ be a stratification of $S$. Without loss of 
	generality, let us have $\dim(S_1)=\Delta_S$. Because of this 
	equality, we will be able to argue about $\Delta_S$ by arguing
	only about $\dim(S_1)$.
 	
	Fix some $s_0\in S_1$ so that the differential of $f_{|S_1}$ at $s_0$ has maximal rank $r$. The set of $s\in S_1$ whose differential has rank $r$ is open, and so there is a neighborhood of $s_0$ on which $f$ has constant rank.  By the  rank theorem \cite{lee2013smooth}, $f$ is locally a projection: this means that  there exists a diffeomorphism $\psi$ which maps an open set $U$ with $s_0\in U \in S_1 $ to $(0,1)^{\Delta_S}$ and a diffeomorphism $\phi:V \to \RR^{D_2} $ where $V$ in open in $\RR^{D_2}$ and contains $f(U)$, such that the function $\tilde f=\phi \circ f \circ \psi^{-1}$ is a projection:
	$$\tilde f (s_1,s_2,\ldots,s_r,\ldots,s_{\Delta_S})=(s_1,s_2,\ldots,s_r,0,0,\ldots,0), \quad \forall (s_1,\ldots,s_{\Delta_S}) \in (0,1)^{\Delta_s} .$$
	For the projection $\tilde f$. we have for every $t$ in the image of $\tilde f$ the equality
	\begin{equation}\label{eq:Delta}
	\Delta_S=r+(\Delta_S-r)=\dim \tilde{f}((0,1)^{\Lambda_s})+\dim \tilde{f}^{-1}(t).
	\end{equation}
	We can now get our result by exploiting the relationship between $f$ and $\tilde f$. Since $\tilde f$ and the restriction of $f$ to $U$ have the same image, we have 
	$$\dim \tilde{f}((0,1)^{\Delta^s})=\dim f(U)\leq \dim(f(S)) $$
	and 
	$$\dim \tilde{f}^{-1}(t)=\dim(\psi({f}^{-1}(\phi^{-1}(t)))\cap U)=\dim(f^{-1}(\phi^{-1}(t))\cap U)\leq \dim f^{-1}(\phi^{-1}(t))\leq \Delta  $$
	Plugging the last two inequalities into the left hand side of \eqref{eq:Delta} concludes the proof.
\end{proof}
We can now prove Theorem~\ref{thm:useful}. Our bundle-based
proof presented below was inspired by ideas in~\cite{balan2006signal}.
	\begin{proof}[Proof of Theorem~\ref{thm:useful}]
		The set
		$$\{(x,y)\in \M \times \M| \, x\not \sim_G y \} $$
  is semi-algebraic as it is the projection onto the $(x,y)$ coordinates of the semi-algebraic set 
  $$\{(x,y,w)| \, p(x,w)\neq p(y,w) \} $$
  It follows that the for every $m\in \NN$ the set 
		\begin{equation}\label{eq:Bm}
		\B_m=\{(x,y,w_1,\ldots,w_m)\in \M\times \M \times \RR^{D_w\times m}| \quad x \not \sim_G y \text{ but } p(x,w_i)=p(y,w_i), i=1,\ldots,m   \} 
		\end{equation}
		is semi-algebraic. We will sometimes refer to the set $\B_m$ as the `bad set'.
		
		Let $\pi$ and $\pi_W$ denote the projections 
		$$\pi(x,y,w_1,\ldots,w_m)=(x,y), \quad 
		\pi_W(x,y,w_1,\ldots,w_m)=(w_1,\ldots,w_m) .$$
		The set $\pi_W(\B_m)$ is precisely the set of $w_1,\ldots,w_m$ which are not separating. Our goal is to show that, when $m$ is big enough, then the dimension of $\pi_W(\B_m)$ is less than $mD_w$ and so is has Lebesgue measure zero.  
		
		Let us start by bounding the dimension of $\B_m$.
		For every $(x,y) \in \pi \B_m$ we have 
		$$\pi^{-1}(x,y)=\{(x,y)\}\times \underbrace{W_{(x,y)}\times W_{(x,y)}\times \ldots \times W_{(x,y)} }_{m \text{ times }}, \text{ where } W_{(x,y)}=\{w\in \RR^{D_w}| \, p(x,w)=p(y,w)\} .$$
		By  assumption $p$ is strongly separating and thus
		$\dim(W_{(x,y)})\leq D_w-1$. Therefore  
		$$\dim \pi^{-1}(x,y)\leq mD_w-m .$$
		It follows from Lemma~\ref{lem:dim}  that 	when $m\geq 2\Dint+1$ 
		\begin{equation}\label{eq:reuse}
		\dim(\B_m)\leq \dim(\pi(\B_m))+mD_w-m\leq 2\Dint-m+mD_w \leq mD_w-1,
		\end{equation}
		and since applying a $\pi_W$ to $\B_m$ can only decrease its dimension we obtain that 
		$\dim(\pi_W(\B_m))\leq  mD_w-1 $
		as required.
	\end{proof}

\subsection{Using the main theorem }\label{sub:method}
The goal of this subsection is to describe in general terms how Theorem~\ref{thm:useful} can be used to achieve low-dimensional orbit-separating invariants (which can be computed efficiently). In the next section we will apply Theorem~\ref{thm:useful} to find a small, efficiently computed collection of separating invariants for several classical group actions, many of which have been studied in the context of invariant machine learning. These results will be presented as a case-by-case elementary analysis, which requires only a combination of Theorem~\ref{thm:useful} with elementary linear algebra arguments. The purpose of this subsection is to provide a general explanation to the results in the next section, based on known results from invariant theory.   

A methodological application of Theorem~\ref{thm:useful} can be achieved by searching for polynomial invariants, and using known results from classical invariant theory which studies these invariants. In particular, for the classical group actions we discuss in the next section, we typically have an available First Fundamental Theorem (FFT) for this group action: that is, a  finite set of invariant polynomials $f_1,\ldots,f_N$ which are called \emph{generators},  such that for every invariant polynomial $p$ 
	there exists a polynomial $q:\RR^N\to \RR$ such that 
	$$p(x)=q(f_1(x),\ldots,f_N(x)) .$$
	The generators of the invariant polynomial ring are \emph{algebraic separators}\footnotemark- that is, any two distinct orbits which can be separated by \emph{any} invariant polynomial will be separated by one of the generators. Let us for now assume that on our 'data manifold' $\M$ the generators do indeed separate orbits (this is often, but not always, the case. We will return to this issue in a few paragraphs). Typically, the
	 cardinality $N$ of the generators is much larger than what we would like. The easiest (but not recommended) method for achieving a smaller collection of polynomials is by starting with a generating set (or some other possibly large known set of semi-algebraic separating invariants ) and
	 applying linear projection.
	 \begin{cor}\label{cor:linear}
	 Let $G$ be a group acting on a semi-algebraic set 
  $\M$ of dimension $\Dint$. Assume that $f_i:\M\to \RR, i=1,\ldots,N$ are semi-algebraic mappings which separate orbits.	Then for almost all $w^{(1)},\ldots,w^{(2\Dint+1)}\in \RR^N$, the functions
	  $$p(x,w^{(j)})=\sum_{i=1}^N w^{(j)}_if_i(x), j=1,\ldots,2\Dint+1$$
	  separate orbits.
	 \end{cor}
 \begin{proof}
 Since $p(x,w)=\sum_{i=1}^N w_if_i(x)$ is polynomial in $w$ for fixed $x$, if we can show that $p$ is a family which separates orbits then it also strongly separates orbits and so Theorem~\ref{thm:useful} gives us separation. Given $x$ and $y$ in $\M$ whose orbits do not intersect, we know that there is some $i$ such that $f_i(x)\neq f_i(y)$ and so $p(x,w=e_i)\neq p(y,w=e_i) $. 
 \end{proof}
	  The dimensionality reduction technique described in Corollary~\ref{cor:linear} is essentially a random linear projection from $\RR^N$ to $\RR^{2\Dint+1}$. This  method was used for generating a small number of separating invariants in \cite{balan2022permutation} and \cite{cahill2020complete}, and is at the heart of the proofs mentioned earlier for the existence of a small set of separating invariants (see \cite{kamke2012algorithmic,dufresne2008separating}). From a computational perspective this approach is sub-optimal as it requires a full computation of all $N$ separating invariants as an intermediate step. 
	
	In many of the examples we discuss in the next section, a significantly more efficient approach is provided by the fact that the generating invariants $f_1,\ldots,f_N$ are obtained from a single invariant via \emph{polarization}. In our context polarization can be described as follows: Assume that $G$ is a subgroup of $GL(\RR^d)$ acting on $\RR^{d\times n}$ and $\RR^{d\times n'}$  by multiplication from the left. If $f:\RR^{d\times n'}\to \RR$ is $G$-invariant, then we can combine $f$, and any linear $W\in \RR^{n \times n'}$,  to create invariants on $\RR^{d\times n}$ of the form
	$$p(X,W)=f(XW), W\in \RR^{n\times n'} $$
	If our original generating invariants $f_1,\ldots,f_N $ were all of the form $f(XW_i), i=1,\ldots,N$, then $p$ is a separating family of semi-algebraic mappings, and so we obtain $2\Dint+1$ separating invariants $p(X,W_i), i=1\ldots,2\Dint+1 $ without needing to compute all of the original generators. For more on the relationship between polarization and separation see \cite{draisma2008polarization}.

	\paragraph{Algebraic separation vs. orbit separation} 
	We now return to discuss a question we touched upon previously: when do invariant polynomials, and thus the generators, separate orbits? In general, a group action can have two distinct orbits which cannot be separated by any invariant polynomial. The main obstruction is that a continuous function which is constant on $G$ orbits is also constant on the orbits' closure. Thus two orbits which do not intersect cannot be separated if their closures do intersect. The classical example for this is the action of $G=\{x>0\}$ on $\RR$ via multiplication. This action has three orbits: positive numbers, negative numbers, and zero. The closures of these orbits all intersect zero and hence the only invariant functions which are continuous on all of $\RR$ are the constant functions. We will find similar issues occurring in the next section for the action on $\RR^{d\times n}$ by scaling or multiplication from the left by $GL(\RR^d)$: in both cases there are no non-constant invariant continuous functions, and  thus we will rely on separating rational invariants for these examples (which are not continuous on all $\RR^{d\times n}$ since they have singularities).  
	
	\footnotetext{in the invariant theory literature the term separation usually refers to what we call here algebraic separation. See \cite{derksen2015computational}}
	
	The scaling group and $GL(\RR^d) $ are open subsets of Euclidean spaces.  For compact groups, the orbits of the group action will be compact and thus equal to their closures, and so the closures of disjoint orbits will remain disjoint. For closed (non-compact) groups acting on $\RR^{d\times n}, d\leq n$,the orbit of every full rank matrix $X$  under the group action is homeomorphic to $G$ and thus closed. Thus for such closed non-compact groups, orbit separation and algebraic separation are identical on the space of $d$ by $n$ full rank matrices which we denote by $\RR^{d\times n}_{full}$. As we will see, when $X$ is not full rank its orbit's closure will often intersect other orbits. 
	
	In the example in the next Section, we will achieve separation of orbits on all of $\RR^{d\times n}$ for actions of compact groups, and separation of orbits on $\RR^{d\times n}_{full}$ for actions of closed non-compact groups. That is, when we can guarantee orbits closures do not intersect, we are able to achieve orbit separation by polynomials. Indeed, for \emph{complex linear reductive groups}, orbits whose closures do not intersect can always be separated by polynomials \cite{derksen2015computational}. These results can be adapted to achieve the separation results we show here: the real groups we discuss are subgroups of complex linear reductive groups, and they share the same set of generators. As such, the separation of orbits for the complex groups implies separation for the real sub-groups. 
	
	We stress again that in practice, the proofs that we use for separation of our continuous family of functions rely only on elementary linear algebra and not on the first fundamental theorem and other invariant theory results noted above. We discuss these results in the next section.

	\section{Separating invariants for point clouds}\label{sec:applications}
	In this section we will use Theorem~\ref{thm:useful} to obtain a collection of $2\Dint+1$ separating invariants (or $\Dint+1$ generically separating invariants) on the data manifold $\M\subseteq \RR^{d\times n}$, for several classical group actions which are of interest in the context of invariant machine learning. For non-compact group actions we will need to assume that $\M$ contains only full rank matrices.  The group actions we consider are multiplications by permutation matrices from the right, or multiplication by the left by: orthogonal transformations, generalized orthogonal transformations, special orthogonal transformations, volume preserving transformations, or general linear transformations. We will also show that these group actions can be combined with translation and scaling with no additional costs. The complexity of computing the invariants is rather moderate, as can be seen in Table~\ref{table} which summarizes the results of this section.
	

	\subsection{Permutation invariance}\label{sub:Sn} 
We begin by considering the action of the group or permutations  on $n$ points, denoted by $S_n$, on $\RR^{d\times n}$ by swapping the  order of the points. This group action has been studied extensively in the recent invariant learning literature (e.g., \cite{qi2017pointnet,segol2019universal,wagstaff2019limitations,zaheer2017deep}). In particular, the approach we suggest here is strongly related to recent results obtained in \cite{balan2022permutation}. This relationship will be discussed in Remark~\ref{remark:balan}. 


Let us first discuss the simple case when $d=1$. Interestingly, in this case the  
ring of polynomials invariants on $\RR^{1 \times n}$ is generated by only $n$ invariants, known as the \emph{elementary symmetric polynomials} . An alternative choice of generators (\cite{kraftinvariant}, exercise 8) which can be computed more efficiently are the power sum polynomials.  
$$\phi_k(x)=\sum_{j=1}^n x_j^k, k=1,\ldots,n .$$
 Let $\Phi:\RR^n\to \RR^n$ denote the map whose coordinates are the power sum polynomials, that is $\Phi(x)=(\phi_1(x),\ldots,\phi_n(x))$. It is known that the power sum polynomials separate orbits (for an elementary proof of this see \cite{zaheer2017deep}).
 
 An alternative way of achieving $n$-dimensional separation is by sorting: let $\sort:\RR^n\to \RR^n $ be the map which sorts a vector in ascending order. This map is invariant to permutations and separates orbits. It is a continuous piece-wise linear map (and so a semi-algebraic map), but is not a polynomial. Note that $\sort(x)$ can be computed in $O(n \log(n))$ operations while computing $\Phi(x)$ requires $O(n^2)$ operations. Additionally, sorting has been successfully used for permutation-invariant machine learning \cite{blondel2020fast,zhang2019deep} while power sum polynomials are discussed as a theoretical tool \cite{zaheer2017deep,segol2019universal} but are not used in practice. Finally, in \cite{balan2022permutation} it is shown that $\sort$ is an isometry (with respect to the Euclidean metric on the output space and a natural metric on the input quotient space $\RR^n/S_n $) while $\Phi$ is not even Bi-Lipschitz .

For $d>1$, separation by polynomials is achievable by   the multi-dimensional power sum polynomials,  defined as 
\begin{equation}\label{eq:powsum}
\phi_\alpha(X)=\sum_{j=1}^n x_j^\alpha, \quad  \alpha\in \ZZ_{\geq 0}^d, |\alpha|\leq n.
\end{equation}
The multi-dimensional power sum polynomials are also generators of the invariant ring. They are used in many papers which prove universality for permutation-invariant constructions \cite{dym2020universality,segol2019universal,zaheer2017deep}. However, the number of power sum polynomials is $n+d\choose n$: when $d>1$ and $n\gg d $ this is significantly larger than the dimension of $\RR^{n\times d}$.

Generalizing the successfulness of the function $\sort$ in separating orbits to the case $d>1$ is less straightforward: it is possible to consider lexicographical sorting: this mapping separates orbits but is not continuous. An alternative generalization could be to sort each row independently- this gives a continuous mapping but it does not separate orbits.
 We now use Theorem~\ref{thm:useful} to propose a low-dimensional set of invariants for the case $d>1$
by polarizing a  $d=1$ separating invariant mapping $\Psi$ (which could for example be  $\sort$ or the 1-dimensional power sums $\Phi $): 

\begin{proposition}\label{prop:Sn}
	Let $\M$ be semi-algebraic subset of $\RR^{d\times n}$ of dimension $\Dint$, which is stable under the action of $S_n$ by multiplication from the right. Let $\Psi:\RR^n \to \RR^n$ be a permutation invariant semi-algebraic mapping which separates orbits, and denote 
	\begin{equation}\label{eq:cont_perm_inv}
	f(X,w^{(1)},w^{(2)})=\langle w^{(2)},\Psi(X^Tw^{(1)})\rangle , X\in \RR^{d \times n},w^{(1)} \in \RR^d, w^{(2)}\in \RR^n 
	\end{equation}
	If $m\geq 2\Dint+1$  then 
	For Lebesgue almost every $(w_1^{(1)},w_1^{(2)}),\ldots,(w_m^{(1)},w_m^{(2)}) $ in $\RR^d\times \RR^{n}$, the invariant functions
	$$f(\cdot,w_i^{(1)},w_i^{(2)}), \quad i=1,\ldots,m$$
	are  separating with respect to the action of $S_n$.
\end{proposition}
\begin{proof}
	The permutation invariance of $f$ for every fixed choice of parameters follows from the invariance of $\Psi$.  By Theorem~\ref{thm:useful} it is sufficient to show that the family of semi-algebraic invariant mappings $f$ strongly separates orbits. Fix some $X,Y\in \RR^{d\times n}$ with disjoint $S_n$ orbits. We need to show that the dimension of the semi-algebraic set 
	$$B=\{(w^{(1)},w^{(2)})\in \RR^d \times \RR^n| f(X,w^{(1)},w^{(2)})=f(Y,w^{(1)},w^{(2)}) \} $$
	is strictly smaller than $n+d$. Since $X$ can not be re-ordered to be equal to $Y$, it follows that the set 
	$$B_1=\{w^{(1)}\in \RR^d | X^Tw^{(1)} \text{ is equal to } Y^Tw^{(1)}\text{ up to reordering} \} $$
	has dimension $d-1$. Thus it is sufficient to show that the set 
	$$\tilde{B}=\{(w^{(1)},w^{(2)})\in \RR^d \times \RR^n| f(X,w^{(1)},w^{(2)})=f(Y,w^{(1)},w^{(2)}) \text{ and } w^{(1)}\not \in B_1  \} $$
	has dimension $\leq n+d-1$.  For fixed $w^{(1)}\not \in B_1$, the orbit separation of $\Psi$ implies that $\Psi(X^Tw^{(1)})\neq \Psi(Y^Tw^{(1)}) $ and so  the set of $w^{(2)}$ for which $\langle w^{(2)}, \Psi(X^Tw^{(1)}) \rangle=\langle w^{(2)}, \Psi(Y^Tw^{(1)}) \rangle $ has dimension $n-1$.  Denoting by $\pi$ the projection of $\tilde{B}$ onto the first coordinate, this means that for every $w^{(1)}\in \pi(\tilde B)$  we have that $\dim(\pi^{-1}w^{(1)})=n-1$ and from Lemma~\ref{lem:dim} this implies 
	$$\dim(\tilde{B})\leq \dim(\pi(\tilde{B}))+n-1\leq n+d-1 $$
	Thus $f$ is strongly separating which concludes the proof.
\end{proof}

We conclude this subsection with some remarks on the significance of this result in the context of the existing literature. Firstly, we note that characterizations of permutation-invariant mappings on $\RR^{d\times n}$ which use separating mappings of the form
$$(x_1,\ldots,x_n)\in \RR^{d\times n}\mapsto \sum_{j=1}^nF(x_j) $$
are common in the literature investigating the expressive power of neural networks for sets and graphs (see for example Lemma 5 in  \cite{morris2021weisfeiler}). However, these are typically based on the \emph{multivariate} power sum polynomials, so that the output dimension of $F$ is the unrealistically high ${n+d \choose n} $ as discussed above. In contrast we can obtain separation on all of $\RR^{d\times n}$ with $2n\cdot d+1 $ invariants, or an even smaller number of invariants when restricting to a lower dimensional $S_n$ stable set $\M $, by choosing $\Psi$ to be the \emph{univariate} power sum polynomial mapping $\Phi$ defined in \eqref{eq:powsum}:
\begin{cor}\label{cor:sum}
	Let $\M$ be semi-algebraic subset of $\RR^{d\times n}$ of dimension $\Dint$, which is stable under the action of $S_n$ by multiplication from the right. Then there exists a polynomial mapping $F:\M \to \RR^{2\Dint+1} $ such that the function 
	\begin{equation}\label{eq:Fseparating}
	\M \ni X=(x_1,\ldots,x_n) \mapsto \sum_{j=1}^n F(x_j)
	\end{equation}
	is invariant and separating.
\end{cor}  
\begin{proof}[proof of corollary]
Denote 
$$\hat \Phi(t)=(t,t^2,\ldots,t^n) $$
so that we have 
$$\Phi(t_1,\ldots,t_n)=\sum_{i=1}^n \hat \Phi(t_i) .$$
Taking $\Psi=\Phi$ in Proposition~\ref{prop:Sn}, we obtain that for $m=2\Dint+1$, and   Lebesgue almost every choice of parameters, the mapping $f(X,w_i^{(1)},w_i^{(2)}), i=1,\ldots,m $ is invariant and separating. Note that the $i$-th coordinate of this map is given by 
\begin{align*}
f(X,w_i^{(1)},w_i^{(2)})&=\langle w_i^{(2)},\Phi(X^Tw_i^{(1)})\rangle=\langle w_i^{(2)},\sum_{j=1}^n\hat \Phi(x_j^Tw_i^{(1)})\rangle\\
&=\sum_{j=1}^n \langle w_i^{(2)},\hat \Phi(x_j^Tw_i^{(1)})\rangle=\sum_{j=1}^n F_i(x_j)
\end{align*}
where we define $F_i:\RR^d \to \RR $ by
$$F_i(x)=\langle w^{(2)},\hat \Phi(x^Tw^{(1)})\rangle.$$
Thus the mapping as in \eqref{eq:Fseparating} with $F=(F_i)_{i=1}^m $ is invariant and separating.
\end{proof}

\begin{remark}\label{remark:balan}
When choosing $\Psi=\sort$ in the formulation of
	 Proposition~\ref{prop:Sn}, we obtain invariants which are closely related to those discussed in \cite{balan2022permutation}. To describe the results in this paper and the relationship to our results here let us first rewrite our results with $\Psi=\sort$ in matrix notations:  Denote by $\colsort:\RR^{n\times m} \to \RR^{n\times m}$ the continuous piece-wise linear function which independently sorts each of the $m$ columns of an $n\times m$ matrix in ascending order. Let $A\in \RR^{d\times m} $ be a matrix whose $m$ columns correspond to $w^{(1)}_1,\ldots, w^{(1)}_m$, and let $B\in \RR^{n \times m} $ be matrix whose $m$ columns correspond to $w^{(2)}_1,\ldots, w^{(2)}_m$. Proposition~\ref{prop:Sn} can be restated in matrix form as saying that  on $\M=\RR^{d\times n}$, for $m=2nd+1$, and  Lebesgue almost every $A,B$, the mapping
	\begin{equation}\label{eq:balan}
		\RR^{d\times n} \ni X\mapsto L_B\circ \beta_A(X)
	\end{equation}
is invariant and separating, where $L_B:\RR^{n\times m} \to \RR^m $ and $\beta_A:\RR^{d\times n} \to \RR^{n \times m} $ are defined by 
$$[L_B(Y)]_j=\sum_{i=1}^n B_{ij}Y_{ij} \text{ and } \beta_A(X)=\colsort (X^TA)  .$$
Note that it follows that $\beta_A$ is invariant and separating as well. 

In \cite{balan2022permutation} Balan et al. consider invariant maps which are compositions of $\beta_A$ as defined above with general linear maps $L:\RR^{n \times m} \to \RR^{2nd} $. Under the assumption that $m>(d-1)n!$, and that the parameters defining $A$ and $L$ are generic, they show that these maps are  separating, and moreover, Bi-Lipschitz (with respect to the Euclidean metric on the output space and a natural metric on the input quotient space $\RR^{d\times n}/S_n $). Thus the main differences between Balan's results and the results here are 
\begin{enumerate}
	\item Balan's proof requires $>n! $ measurements to guarantee separation of $\beta_A$, while we only require $2nd+1$ measurements. 
	\item We consider compositions of $\beta_A$ with sparse linear mappings $L_B$ (these same mappings are  suggested in \cite{zhang2019fspool}). In contrast, Balan considers general linear mappings $L$ which are defined by  $n$ times more parameters than $L_B$.  
	\item Balan's results show that $\beta_A $ and $L\circ \beta_A$ are Bi-Lipschitz. We do not consider this important aspect in this paper. We note that Balan shows that $\beta_A$ is Bi-Lipschitz whenever $\beta_A$ is separating. Thus their results coupled with out own show that $\beta_A$ is Bi-Lipschitz even when $m=2nd+1 $. The Bi-Lipschitzness of our sparse $L_B$ was not directly addressed in \cite{balan2022permutation}, and we leave this to future work.     
\end{enumerate}

%
%
\end{remark}
		\subsection{Orthogonal invariance}\label{sub:Od}
		We now consider the action of the group or orthogonal matrices $O(d) $  on $\RR^{d\times n}$ via multiplication from the left.
 
 We consider a polynomial family of invariants of the form	\begin{equation}\label{eq:pr}
		p(X,w)=\|Xw\|^2, X\in \RR^{d \times n}, w\in \RR^n .
	\end{equation}
For fixed $X,w$ the cost of computing this invariant is $\bigO(n\cdot d) $. This choice of invariants is a natural generalization of the type of invariants encountered in phase retrieval (see discussion in Subsection~\ref{sub:related} and Remark~\ref{remark:phase}). It also can be seen as a realization of the invariant theory based methodology discussed in \ref{sub:method}. The ring of invariant polynomials is generated by the inner product polynomials $\langle x_i,x_j\rangle, 1\leq i\leq j \leq n$ \cite{weyl1946classical}. It is thus also generated by the polynomials 
	\begin{equation}\label{eq:norm_gen}
	\|x_i\|^2, i=1,\ldots,n \text{ and }  \|x_i-x_j\|^2 1\leq i<j\leq n \end{equation}
	since these polynomials have the same linear span as the inner product polynomials. These invariant are obtained from the squared norm invariant on $\RR^d$ by polarization, and so are all of the form \eqref{eq:pr} for an appropriate choice of $w\in \RR^n$, that is 
	$$p(X,w=e_i)=\|x_i\|^2 \text{ and } p(X,w=e_i-e_j)=\|x_i-x_j\|^2  .$$
Our result is now an easy consequence of the discussion so far and Theorem~\ref{thm:useful}:
	\begin{proposition}\label{prop:Od}
		Let $n\geq d$,  let $\M$ be a  semi-algebraic subset of $\RR^{d\times n}$ of dimension $\Dint$, which is stable under the action of $O(d)$. If $m\geq 2\Dint+1$ then 
		For Lebesgue almost every $w_1,\ldots,w_m $ in $\RR^n$ the invariant polynomials 
		$$X\mapsto \|Xw_i\|^2, \quad i=1,\ldots,m $$
		are  separating with respect to the action of $O(d)$.
	\end{proposition}
	\begin{proof}
		By Theorem~\ref{thm:useful} it is sufficient to show that the family of invariant functions $p$ is strongly separating, and as they are polynomials we only need to show separation. It is sufficient to show that the finite collection of polynomials in \eqref{eq:pr} are separating, which as mentioned above is equivalent to showing that the inner product polynomials $\langle x_i,x_j \rangle $ are separating. This is just the known fact that the Gram matrix $X^TX $ determines $X$ uniquely up to orthogonal transformation. see e.g., Lemma~\ref{lem:isometries} and its proof in the appendix.
	\end{proof}

	We note that Proposition~\ref{prop:Od} (with a slightly smaller number of separating invariants) can also be deduced immediately from Theorem 4.9 in \cite{rong2021almost} which discusses the equivalent problem of separating rank one matrices using rank one linear measurements.  
	\subsection{Special orthogonal invariance}\label{sub:SOd}
	We now turn to the action of the special orthogonal group $SO(d)=\{R\in O(d),det(R)=1\}$ on $\RR^{d\times n}$ by multiplication from the left. 
	The invariant ring for this group action is generated by the polynomials in \eqref{eq:norm_gen} together with the invariant polynomials  
	\begin{equation}\label{eq:det}
	[i_i,i_2,\ldots,i_d](X)=det(x_{i_1},x_{i_2},\ldots,x_{i_d}), \quad 1\leq i_1<i_2<\ldots<i_d\leq n .
	\end{equation}
	Accordingly, the generators can all be realized by specific choices of $(w,W)$ from the family of polynomial invariants 
	\begin{equation}\label{eq:sop}
	p(X,w,W)=\|Xw\|^2+det(XW), X\in \RR^{d\times n}, w\in \RR^n, W\in \RR^{n\times d} .
	\end{equation}
The complexity of calculating  each invariant (for fixed $w,W$) is dominated by the matrix product $XW$ which with the standard method for matrix multiplication requires $\bigO(n\cdot d^2)$ operations. 
	We can easily prove that this family of invariants separates orbits:
	\begin{proposition}
		Let $n\geq d$, and let $\M$ be a  semi-algebraic subset of $\RR^{d\times n}$ of dimension $\Dint$, which is stable  under the action of $SO(d)$. If $m\geq 2\Dint+1$ then 
		For Lebesgue almost every $(w_1,W_1),\ldots,(w_m,W_m) $ in $\RR^n\times \RR^{n\times d}$, the invariant polynomials 
		$$X\mapsto \|Xw_i\|^2+det(XW_i), \quad i=1,\ldots,m $$
		are separating with respect to the action of $SO(d)$.
	\end{proposition}
	\begin{proof}
		By Theorem~\ref{thm:useful} it is sufficient to show that the family of invariant functions $p$ is strongly separating, and as they are polynomials we only need to show separation. Let $X,Y\in \RR^{d\times n}$ which do not have the same orbit. If $X$ and $Y$ are not related by any orthogonal transformation then we already showed that they can be separated by the `norm polynomials'. 	We now need to consider the case where $X$ and $Y$ are not related by a rotation, but are related by $X=RY$ where $R$ is orthogonal with   $det(R)=-1$. In this case we see that $X$ and $Y$ have the same rank. Moreover, they must be full rank, since otherwise we can multiply $R$ by an orthogonal transformation $R_0$ with $det(R_0)=-1$ which fixes the column span of $Y$ and obtain $X=RR_0Y$ and $det(RR_0)=1$, in contradiction to the fact that $X$ and $Y$ do not have the same $SO(d)$ orbit. Since $X$ is full rank, we can choose $1\leq i_1<\ldots<i_d\leq n$ such that $[i_1,\ldots,i_d](X)\neq 0$. This polynomial will separate $X$ and $Y$ since
		$$-[i_1,\ldots,i_d](X)\neq [i_1,\ldots,i_d](X)=[i_1,\ldots,i_d](RY)=-[i_1,\ldots,i_d](Y) $$
		
	\end{proof}

	\begin{remark}\label{remark:phase}
	For $d=2$ there are more efficient invariants than the ones we suggest here: as mentioned previously, known results \cite{conca2015algebraic,balan2006signal} on complex phase retrieval state that for generic $m=4n-4$ complex vectors $w^{(1)},\ldots,w^{(m)}$ in $\CC^n$, the maps 
	\begin{equation}\label{eq:s1}\CC \ni z\mapsto |\langle z, w_j \rangle | 
		\end{equation}
	separate orbits of the action of $S^1$ on $\CC^n$. Note that the linear map  $\langle z,w_j \rangle  $ is $S^1$ equivariant, that is 
	$$\langle \xi z,w_j \rangle=\xi \langle  z,w_j \rangle, \forall \xi \in S^1  $$
	 Identify $\CC^n \cong \RR^{2\times n}$ and $S^1\cong SO(2) $, we see that there are $m$ linear $SO(2) $ equivariant  maps $W^{(1)},\ldots,W^{(m)}$ from $\RR^{2\times n}$ to $\RR^2$, modeling multiplication in $\CC^1$, and  
	\begin{equation}\label{eq:so2}
	\RR^{2\times n} \ni X \mapsto \|XW^{(j)}\| \end{equation}
	is $SO(2)$ invariant and separate orbits. Each one of these linear maps $W^{(j)} $ is parameterized by $2n$ real numbers, while our invariants in \eqref{eq:sop} are parameterized by $3n$ parameters (when $d=2$).
	
	When $d\neq2$, it would be natural to look for separating invariants of the form \eqref{eq:so2}, where the $W^{(j)}$ are $SO(d)$ equivariant linear maps from $\RR^{d\times n}$ to $\RR^d$, and avoid the additional determinant term we use in \eqref{eq:sop}. However, in Proposition~\ref{prop:equiv} in the appendix we show that when $d\neq 2$, the only linear $SO(d)$ equivariant maps   are of the form $X\mapsto Xw$ with $w\in \RR^n$. These maps are also $O(d)$ equivariant, and as a result $\|Xw\|$ is $O(d)$ invariant. It follows that these maps cannot separate point clouds which are related by reflections but do not have the same orbit in $SO(d)$.      
	\end{remark}
	
	\subsection{Isometry groups for non-degenerate bi-linear forms}
	The next examples we consider are isometry groups for non-degenerate bi-linear forms. As usual we assume $n\geq d$, and we are given a symmetric invertible $Q\in \RR^{d\times d}$ which induces a symmetric bi-linear form 
	$$\langle x,Qy\rangle, x,y\in \RR^d.$$ 
	We define a  $Q$-isometry as a matrix $U\in \RR^{d\times d}$ such that $U^TQU=Q$, and thus the symmetric bi-linear form defined by $Q$ is preserved by $U$:
	$$\langle Ux,QUy\rangle=\langle x,U^TQUy\rangle=\langle x,Qy\rangle $$
	The set of $Q$-isometries is a subgroup of $GL(\RR^d)$ which we denote by $O_Q(d)$. The orthogonal group $O(d)$ discussed earlier corresponds to $Q=I_d$. Indefinite orthogonal groups $O(s,d-s) $ correspond to diagonal $Q$ matrices with $s$ positive unit entries and $d-s$ negative unit entries. In particular the Lorenz group $O(3,1)$ (together with translations) is an important symmetry group in special relativity, and has been discussed in the context of invariant machine learning for physics simulations \cite{villar2021scalars,bogatskiy2020lorentz}. 

    We consider the task of finding separating invariants for the action of $O_Q(d)$ on $\RR^{d\times n}$ by multiplication from the left. A natural place to start is the $Q$ Gram matrix $X^TQX $, whose coordinates are  the $Q$-inner products $\langle x_i,Qx_j\rangle, 1\leq i\leq j \leq n $. Indeed, at least for the Lorenz group it is known that the inner product polynomials are indeed generators \cite{villar2021scalars}. When $Q$ is positive definite, the $Q$-inner products do indeed separate orbits. When $Q$ is not positive definite, this is no longer always true: consider the following example with $d=2, n=2$ and  
	\begin{equation*}
	Q=\begin{bmatrix}
	1 & 0\\0& -1
	\end{bmatrix},
\quad 
X=\begin{bmatrix}
	0 & 0\\0& 0
\end{bmatrix},
\quad 
Y=\begin{bmatrix}
	1 & 1\\1& 1
\end{bmatrix},
\quad 
	\end{equation*} 
We see that the $Q$-Gram matrix of $X$ and $Y$ is both zero, while $X$ and $Y$ cannot be related by a $Q$-isometry since $Q$-isometries are invertible. However, the $Q$-inner products do separate orbits when restricted to the set of full rank matrices  $\RR^{d\times n}_{full}$.  The following lemma, proved in the appendix, formulates this claim. The proof is essentially taken from \cite{gortler2014generic} corollary 8. 
	\begin{lemma}\label{lem:isometries}
		Assume $Q\in \RR^{d\times d}$ is a symmetric invertible matrix and $X,Y\in \RR^{d\times n}$ have the same $Q$-Gram matrix. If (i) $X$ has rank $d$ or (ii) $Q$ is positive definite,
then $X$ and $Y$ are related by a $Q$-isometry.
\end{lemma}

 Once we know that the $Q$ inner products separate orbits on $\RR^{d\times n}_{full}$, we proceed as we did for $O(d)$. We see that the `Q-norm polynomials' 
$$\langle x_i,Qx_i \rangle, i=1,\ldots,n \text{ and } \langle x_i-x_j,Q(x_i-x_j) \rangle, 1\leq i <j\leq n  $$
span the $Q$-inner product polynomial and hence are also separating on $\RR_{full}^{d\times n}$. We can then prove an analogue of Proposition~\ref{prop:Od} using the same arguments used there:

	\begin{proposition}\label{prop:OQd}
		Let $n\geq d$, let $Q\in \RR^{d\times d}$ be symmetric and invertible, and let $\M$ be a  semi-algebraic subset of $\RR^{d\times n}_{full}$ of dimension $\Dint$, which is stable under the action of $O_Q(d)$. If $m\geq 2\Dint+1$  then 
		For Lebesgue almost every $w_1,\ldots,w_m $ in $\RR^n$ the invariant polynomials 
		$$X\mapsto \langle Xw_i,QXw_i\rangle, \quad i=1,\ldots,2nd+1 $$
		are  separating with respect to the action of $O_Q(d)$.
	\end{proposition}
	\subsection{Special linear invariance}\label{sub:SL}
We now give a full treatment for the group action described in Example\ref{ex:SLd}. 	We consider the action of the special linear group $SL(d)=\{A\in \RR^{d\times d}| \, \det(A)=1 \} $ 
	on $\RR^{d \times n}$ by multiplication from the left. The generators for the ring of invariants is given by the determinant polynomials \cite{kraftinvariant}
\begin{equation}\label{eq:det2}
[i_1,\ldots,i_d](X)=\det(x_{i_1},\ldots,x_{i_d}) , 1\leq i_1<i_2<\ldots<i_d\leq n ,
\end{equation}
	which we have already encountered in Subsection~\ref{sub:SOd}. The generators  cannot separate matrices in $\RR^{d\times n}$ which are not full rank, since for such matrices we will always get zero determinants. In Proposition~\ref{prop:special} in the appendix we give an elementary proof that the determinant polynomials from \eqref{eq:det2} separate orbits on $\RR_{full}^{d \times n}$. 
	
	The separation of the determinant polynomials in \eqref{eq:det2} together with Theorem~\ref{thm:useful} implies
	\begin{proposition}\label{prop:SL}
		Let $n\geq d$, and let $\M$ be a  semi-algebraic subset of $\RR^{d\times n}_{full}$ of dimension $\Dint$, which is stable under the action of $SL(d)$. If $m\geq 2\Dint+1$  then 
		For Lebesgue almost every $W_1,\ldots,W_m $ in $\RR^{n\times d}$ the invariant polynomials 
		$$X\mapsto \det(XW_i), \quad i=1,\ldots,m $$
		are separating with respect to the action of $SL(d)$.
	\end{proposition}

\subsection{Translation}
We consider the action of $\RR^d$ on $\RR^{d\times n}$ by translation: 
$$t_*(X)=X+t1_n^T. $$
 We can easily compute $n\cdot d$ separating invariants for this action: examples include the mapping $ X \mapsto X-x_11_n^T $ suggested in \cite{villar2021scalars}  or the centralization mapping $cent(X)= X-\frac{1}{n}X1_n1_n^T$. The centralization mapping is equivariant w.r.t the action of multiplication by a matrix $A\in GL(\RR^d)$ from the left, and a permutation matrix $P$ from the right. That is 
$$cent(AXP^T)=AXP^T-\frac{1}{n}AXP^T1_n1_n^T=AXP^T-\frac{1}{n}AX1_n1_n^T=AXP^T-\frac{1}{n}AX1_n1_n^TP^T=Acent(X)P^T .$$
It follows that if $f:\RR^{d\times n}\to \RR^m$ is invariant with respect to some group $G$ which is a subgroup of $GL(\RR^d)\times S_n$, then $f(cent(X))$ will be invariant with respect to the group $\langle G,\RR^d \rangle  $ generated by $G$ and the translation group. Additionally, if $f$ separates orbits w.r.t the action of $G$, then $f$ separates orbits with respect to the action of $\langle G,\RR^d \rangle  $. To see this, note that if $X,Y\in \RR^{d\times n}$ and $f(cent(X)=f(cent(Y)) $, then since $f$ separates orbits we have that there exist some $(A,P)\in G\leq GL(\RR^d)\times S_n$ such that $cent(X)=Acent(Y)P^T $, and so $X$ is obtained from $Y$ by translation by the mean of $Y$, follows by a $G$ action, and translation by the mean of $X$.    
  
\subsection{Scaling}
We consider the action of $\RR_{>0}=\{x>0\}$ on $\RR^{d\times n}$ by scaling (scalar-matrix multiplication). In this case there are no non-constant invariant polynomials, or in fact any non-constant invariants which are continuous on all of $\RR^{d\times n}$. This is because the orbit of each $X\in \RR^{d\times n}$ contains the zero matrix $0\in \RR^{d\times n}$ in its closure. However, we can easily come up with non-polynomial separating invariants with singularities at zero, such as $X\mapsto \|X\|^{-1} X$, where $\|\cdot \|$ denotes some norm on $\RR^{d\times n}$. If we choose the Frobenius  norm $\|\cdot\|_F$, this mapping is equivariant with respect to multiplication by an orthogonal matrix from the left and a permutation matrix from the right. As a result, if $f:\RR^{d\times n}\to \RR^m$ is invariant with respect to some group $G$ which is a subgroup of $O(d)\times S_n$, then $X\mapsto f(\|X\|_F^{-1} X)$ will be invariant with respect to the group generated by $G$ and the scaling group. Additionally, if $f$ separates orbits with respect to the $G$ action, then $X\mapsto f(\|X\|_F^{-1} X)$ separates orbits with respect to the group generated by $G$ and the scaling group. 

\subsection{General Linear invariance}
We consider the problem of finding separating invariants for the action of the general linear group $GL(\RR^d) $ on $\RR^{d\times n}$ by multiplication from the left. There are no non-constant  polynomial invariants for this action, since this is the case even for the scaling group which  is a subgroup of $GL(\RR^d)$. We consider a family of rational invariants 
$$q(X,W)=\frac{det^2(XW)}{det(XX^T)}, \quad X\in \RR^{d\times n}, W\in \RR^{n\times d}$$
The function $q$ is well defined on  $\RR^{d\times n}_{full} \times \RR^{n\times d}$, and for fixed $W$ the function $X\mapsto q(X,W)$ is $GL(\RR^d) $-invariant. We prove: 

\begin{proposition}\label{prop:GL}
	Let $n\geq d$,  let $\M$ be a  semi-algebraic subset of $\RR_{full}^{d\times n}$ of dimension $\Dint$, which is stable under the action of $GL(\RR^d)$. If $m\geq 2\Dint+1$, then 
	for Lebesgue almost every $W_1,\ldots,W_m $ in $\RR^{n\times d}$ the invariant polynomials 
	$$X\mapsto q(X,W_i), \quad i=1,\ldots,m $$
	are  separating with respect to the action of $GL(\RR^d)$.
\end{proposition}
\begin{proof}
	By Theorem~\ref{thm:useful} it is sufficient to show that the family of rational functions $q$ is strongly separating. In fact since $q(X,W)$ is polynomial in $W$ for every fixed $X$, it is sufficient to show orbit separation.  
	
	Let $X,Y\in \RR_{full}^{d\times n}$ be two full rank point clouds whose orbits do not intersect.
	Since $X$ is full rank, it has $d$ columns which are linearly independent, for simplicity of notation we assume these are the first $d$ columns. If the first $d$ columns of $Y$ are not linearly independent, then by choosing $W_0=[I_d, \,  0]^T$  we get that
	$$q(X,W_0)= \frac{\left[det(x_1,\ldots,x_d) \right]^2}{det(XX^T)}$$
	is zero on $Y$ and not on $X$ and so it separates the two points. Thus we can assume that the first $d$ columns of $Y$ are linearly independent. It follows that the matrix $A$ defined uniquely by the equations $Ax_i=y_i, i=1,\ldots,d$ is non-singular.
	
	By assumption, $AX\neq Y$ so  there exists some index $j, d<j\leq n$ such that $Ax_j\neq  y_j$. Since $y_1,\ldots,y_d$ span $\RR^d$, there exist $\alpha_1,\ldots,\alpha_d$ and $\beta_1,\ldots,\beta_d$ such that 
	$$Ax_j=\sum_{i=1}^d \alpha_iy_i, \quad  y_j=\sum_{i=1}^d \beta_iy_i  $$
	and since $Ax_j\neq y_j$ there exists some $k, 1\leq k \leq d$ such that $\alpha_k\neq \beta_k$. Let $W_1\in \RR^{n\times n}$ be a matrix such that for all $Z=[z_1,\ldots,z_n]\in \RR^{d\times n}$ we have 
	$$ZW_1=[z_1,z_2,\ldots,z_{k-1},\beta_k z_k-z_j,z_{k+1},\ldots,z_n] .$$
	Then the first $d$ columns of $YW_1$ have rank $d-1$, while the first $d$ columns of $AXW_1$, and therefore also of $XW_1$, have rank $d$. It follows that $$q(X,W_1W_0)=q(XW_1,W_0)\neq 0= q(YW_1,W_0)=q(Y,W_1W_0).$$
	Thus we have shown that $q$ separates orbits.  
\end{proof}

\subsection{Intractable separation for permutation actions on graphs}\label{sub:graphs}
	Consider the action of the permutation group $S_n$ on the vector space $\RR^{n\times n}$ by conjugation: given a permutation matrix $P\in S_n$, and a matrix $X\in \RR^{n\times n} $, this action is defined as 
$$P_*X=PXP^T .$$
If $A$ is an adjacency matrix of a graph, applying a relabeling $\sigma  $ to the vertices, creates a new graph, isomorphic to the previous one, whose adjacency matrix  $A'$ is equal to $A'= PAP^T $ for $P$ the matrix representation of the permutation $\sigma$. We are thus interested in studying this action on the set of weighted adjacency matrices $\Mweight$ defined as
\begin{equation}\label{eq:Mweight}
\Mweight=\{A\in \RR^{n\times n}| A=A^T, A_{ii}=0 \text{ and } A_{ij}\geq 0, \forall i,j=1,\ldots,n\}.
\end{equation}

We note that $\Mweight$ is stable under the action of $S_n$, and has  dimension $\nweight=(n^2-n)/2 $. 

More generally, we will want to think of this action of $S_n$ on $S_n$ stable semi-algebraic subsets $\M$ of $\Mweight$ of arbitrary dimension $\Dint$. For example, the collection of all binary (unweighted) graphs can be parameterized by the finite $S_n$ stable set
	$$\Mbinary=\{A\in \Mweight| \, A_{ij}\in \{0,1 \}, \forall i,j=1,\ldots,n \}.$$
Another natural example includes (weighted or unweighted) graphs of bounded degree. 

Let us now consider the task of constructing separating invariants for the action of $S_n$ on a semi-algebraic stable subset $\M\subseteq \Mweight$ of dimension $\Dint$ . As our discussion suggest, we will be able to find such separating invariants of dimension $2\Dint+1$. However, the computational effort involved in computing the invariants in our constructions grows superpolynomially in $n$. This is not surprising as a polynomial time algorithm for computing separating invariants for the action of $S_n$ on $\Mbinary$, will lead to a polynomial time algorithm for the notoriously hard Graph Isomorphism problem (see \cite{grohe2020graph}).

One simple separating family for the action of $S_n$ on $\Mweight$ is polynomials of the form 
$$p(X,W)=\prod_{P\in S_n} \|PXP^T-W\|_F^2, \, X\in \Mweight, W\in \RR^{n\times n} .$$
Clearly for fixed $W$ the polynomials $X\mapsto p(X,W) $ is permutation invariant, and separation follows from the fact that if $X,Y \in \Mweight $ and $X \not \sim Y $, then taking $W=X$ we obtain
$$p(X,W)=0\neq p(Y,W) .$$
Thus by Theorem~\ref{thm:useful}, we can obtain $m=2\Dint+1$  separating invariants for the action of $S_n$ on $\M$,    as for  Lebesgue almost every $W_i, i=1,\ldots,m $ in $\left(\RR^{n \times n}\right)^m $, the functions 
$$X \mapsto p(X,W_i), i=1,\ldots,m $$
are invariant and separating. Note however that the degree of these polynomials is $2\cdot n!$ and so computing these invariants is not tractable.

	\section{Generic separation}\label{sec:generic}
	Generic separation is a relaxed notion of separability which is often easier to achieve than full separation:
	\begin{definition}[Generic separation]\label{def:generic}
		Let $G$ be a group acting on a semi-algebraic set $\M$. Let $\Y$ be a  set, and let $f:\M\to \Y$ be an invariant function. We say that $f$ is generically separating on $\M$, with singular set $\N$, if   $\N\subseteq \M$ is a semi-algebraic set which is stable under the action of $G$, satisfies $\dim(\N)<\dim{\M}$, and for every $x\in \M \setminus \N $, if there exists some $y\in \M$ such that $f(x)=f(y) $, then $x \sim_G y $.  
\end{definition}
Note that being generically separating on $\M$ is slightly stronger than being separating on $\M \setminus \N$, since the latter would only consider $Y\in\M \setminus \N$.

 Some possible practical disadvantages of generic separating invariants in comparison to fully separating invariants were discussed in Subsection~\ref{sub:main_results}. Our purpose in this section is to show that achieving generic separation is easier than achieving full separation in two respects:
 \begin{enumerate}
 	\item While $2\Dint+1 $ separating invariants can be obtained by randomly choosing parameters of families of strongly separating invariants, when considering generic separation $\Dint+1 $ invariants suffice. We discuss this next in Subsection~\ref{sub:general_generic}.
 	\item More importantly, for some group actions it is easy to come up with generic separators which can be computed efficiently, but obtaining true separators with low complexity seems inaccessible. This is discussed in Subsection~\ref{sub:harder}.
 \end{enumerate} 

\subsection{Generic separation from generic families of separators}\label{sub:general_generic}
In this section we prove an analogous theorem to  Theorem~\ref{thm:useful} where now we discuss generic invariants. The notion of generic separating invariants was defined in Definition~\ref{def:generic}. We now define this notion for families of invariants: 
\begin{definition}[Strong generic separation for invariant families]	
	Let $G$ be a group acting on a semi-algebraic set $\M$. We say that a family of semi-algebraic functions $p:\M\times \RR^{D_w} \to \RR$ \emph{strongly} separates orbits generically on $\M$ ,with respect to a singular set $\N$, if  $\N\subseteq \M$ is a semi-algebraic set which is stable under the action of $G$, satisfies $\dim(\N)<\dim{\M}$, and for every $x\in \M \setminus \N $ and $y\in \M$ with $x \not \sim_G y $, the set 
		$$\{w\in \RR^{D_w}| \, p(x,w)=p(y,w)\} $$
		has dimension $\leq D_w-1 $.
	\end{definition}
We can now state an analogous theorem to  Theorem~\ref{thm:useful} for separating invariants. As mentioned above, the cardinality for generic separation is $\Dint+1$ and not the $2\Dint+1$ we have in Theorem~\ref{thm:useful} for full separation. 
		\begin{theorem}\label{thm:generic}
		Let $G$ be a group acting on a semi-algebraic set $\M$ of dimension $\Dint$. Let $p:\M\times \RR^{D_w} \to \RR$ be a  family of $G$-invariant semi-algebraic functions. 
		If $p$ strongly separates orbits \emph{generically} in $\M$, then for Lebesgue almost-every $w_1,\ldots,w_{\Dint+1}\in \RR^{D_w} $ the $\Dint+1$ $G$-invariant semi-algebraic functions
		$$p(\cdot,w_i), i=1,\ldots,\Dint+1 $$
		generically separate orbits in $\M$.
	\end{theorem}
	
	\begin{proof}[Proof of Theorem~\ref{thm:generic}]
	 Similarly to  the proof of Theorem~\ref{thm:useful}, we can consider the `bad set'
		\begin{equation*}
			\B_m=\{(x,y,w_1,\ldots,w_m)\in (\M\setminus \N)\times \M \times \RR^{D_w\times m}| \quad x \not \sim_G y \text{ but } p(x,w_i)=p(y,w_i), i=1,\ldots,m   \} 
		\end{equation*}
	and repeat the dimension argument used there, together with our requirement that $m\geq \Dint+1 $, to obtain 
	$$\dim(\B_m)\leq 2\Dint+m(D_w-1)\leq \Dint+mD_w-1 . $$
Denote 
$$W=(w_1,\ldots,w_m) \text{ and } \pi_W(x,y,W)=W.$$
Our goal next is to bound the dimension of the fiber $\pi_W^{-1}(W) $ over $W$.
		Let $U_1,\ldots,U_K$ be a stratification of $\B_m$ so 
  that each $U_k$ is a manifold and $\cup_{k=1}^K U_k=\B_m$. 
  For every fixed $k$, if the dimension of $\pi_W(U_k)$ is less 
  than $mD_w$ then almost all $W$ will not be in the 
  projection and so the intersection of the fiber over these 
  $W$ with $U_k$ will be empty.  Now let us assume that the 
  dimension of $\pi_W(U_k)$ is  $mD_w$. By Sard's theorem 
  \cite{lee2013smooth}, almost all $W$ in $\pi_W(U_k)$ is a 
  regular value of the restriction of  $\pi_W$ to $U_k$. By the 
  pre-image theorem \cite{tu2011manifolds}, every regular value $W$ is either not in the image, or the dimension of its fiber 
  $\pi_W^{-1}(W)\cap U_k $ is precisely 
  $$\dim(U_k)-\dim(\RR^{D_w\times m})\leq \dim(\B_m)-mD_w\leq \Dint-1 $$
		It follows that for almost all $W=(w_1,\ldots,w_m)$,  the fiber over $W$
		$$\pi_W^{-1}(W)=\cup_{k=1}^K\left( \pi_W^{-1}(W)\cap U_k \right) $$
		has dimension strictly smaller than $\Dint $. Thus this is also true for the projection of the fiber onto the $x$ coordinate, which is the set 
		\begin{align*}
		\N_W&=\{x\in (\M\setminus \N)| \exists y \in \M \text{ such that } (x,y,W)\in \pi_W^{-1}(W)\}\\
		&=\{x\in (\M\setminus \N)| \exists y \in \M \text{ such that }x \not \sim_G y \text{ but } p(x,w_j)=p(y,w_j), \forall j=1,\ldots,m\} 
		\end{align*}
		it follows that for such $W=(w_1,\ldots,w_m)$, the invariants $p(\cdot,w_i), i=1,\ldots,m $ are generically separating on $\M$ with singular set $\N \cup \N_W $. 
	\end{proof}
	
	\subsection{Generic separation for graphs}\label{sub:harder}
	We now return to discuss the graph separation result we discussed in Subsection~\ref{sub:graphs}. We will show that while computing true separating invariants on $\Mweight$ in polynomial time seems out of reach, generically separating invariants can be computed in polynomial time. We note that this is hardly surprising: the fact that graph isomorphism is not hard for generic weighted \cite{dym2018exact} or unweighted \cite{babai1982isomorphism} graphs is well known.   
\begin{proposition}
For every natural $n\geq 2$, the mapping 
$$F(A)=\left(\sort(A1_n),\sort\left\{A_{ij}, 1\leq i <j \leq n \right\}  \right) $$
is generically separating and invariant with respect to the action of $S_n$ on $\Mweight $.
\end{proposition}
\begin{proof}
When $n=2$ a matrix in $\Mweight$ is determined by a single off-diagonal element and thus $F$ is easily seen to be (globally) separating. 

Now assume $n>2$. A generic $A\in \Mweight $ has the following property: the summation of any two different subsets of the sub-diagonal element of $A$ yields a different result. In particular, (i) two different rows of  $A$ sum to different values and (ii) the sub-diagonal elements are pairwise distinct. 

Now let $B$ be a graph in $ \Mweight $ with $F(A)=F(B) $. Since $\sort(A1_n)=\sort(B1_n)$, we can without loss of generality assume that  the vertices of $B$ are ordered so that all row sums of $A$ and $B$ agree. We now claim that $A=B$. Indeed, we know that $\sort\left\{A_{ij}, 1\leq i <j \leq n \right\}=\sort\left\{B_{ij}, 1\leq i <j \leq n \right\}$ and if there is some $i<j$ for which $A_{ij}\neq B_{ij} $, then either the $i$-th of $j$-th row of $B$ consists solely of elements of $A$, but does not contain $A_{ij}$, and by assumption on $A$ this yields a contradiction to the fact that the rows of $A$ and $B$ sum to the same number.    
\end{proof}

\section{Computer numbers}\label{sec:computer_numbers}

Theorem~\ref{thm:useful} is a statement about ``almost every'' choice for 
$w_1,\ldots,w_{m}\in \RR^{D_w}$ in a measure theoretic sense,
over the reals. A slight disadvantage of these separators is that while we have separation for almost all $w_1,\ldots,w_{m}$, we cannot point at any specific $w_1,\ldots,w_{m}$ for which we can absolutely guarantee separation. We do not address this
difficult problem in this paper.

A second issue is that  in the computational setting, where
each $w_i$ is represented using a finite number of bits, it is conceivable, a-priori that
almost all real $w_1,\ldots,w_{m}$ are separating, any yet all $w_1,\ldots,w_{m}$ which can be represented as computer numbers are not separating. Our goal is to show that even when    the $w_i$ are represented with a finite number of bits, we can still obtain separation with high probability.

Our strategy for obtaining this goal is as follows: we first show that all bad $w_1,\ldots,w_m$ are 
in the zero-locus of some polynomial $f$ whose degree is at most $R$. We then will use the Schwartz-Zippel lemma:
\begin{lemma} (Schwartz, Zippel, DeMillo, Lipton). 
	Let $f$ be
	a non-zero 
	real polynomial of degree $R$
	in $l$ variables over a field $k$. 
	Select $x_1 ,...,x_l$ uniformly at random from a
	finite subset $X$ of $k$. 
	Then the probability than $f(x_1 ,...,x_l)$ is 0 is less than $R/|X|$.
\end{lemma}
This lemma implies  that if we use more than $\log_2 (\epsilon^{-1}R)$ bits to represent
our real numbers, so that they are selected from a finite alphabet $X$ with $>\epsilon^{-1}R $ elements, then the probability of picking a bad collection of $(x_1,\ldots,x_l)$  is less than $\epsilon$.  In our setting, the $x_i$ will comprise the coordinates of our $(w_i,...,w_m)$.

The challenge in this strategy is finding an appropriate non-zero $f$ with bounded degree. 
In the following we will replace the real semi-algebraic requirements for Theorem~\ref{thm:useful} with the stronger algebraic requirement:
\begin{lemma}\label{lem:tao}
	Let $G$ be a group acting  on $\RR^D$,
  Let $p:\RR^D \times \RR^{D_w} \to \RR$ be a  family of $G$-invariant separating polynomials of degree $r$, and let  $m\ge 2D+1$. Then the `bad set' of $(w_i,\ldots,w_m)$ which do not form a separating mapping is contained in the zero-set of a non-zero polynomial $f$ of degree at most $(mD_w-1)(2D+mD_w-m) (r^m)$.
\end{lemma}
Applying Scwartz-Zippel to this bound we see that we need
$O(D\log(r))$ bits to obtain a formal proof that we are unlikely to
pick a bad set of $w_i$.   Note that in our applications,
$r$ is quite low.

\begin{proof}[Proof (sketch) of Lemma~\ref{lem:tao}]

We will first move to the complex setting, where can use the Bezout's theorem. Then we will restrict our attention to the real locus.
To move to the complex setting, we now
think of $p(x,y,w)$ as a polynomial over complex variables. 
We define $(x \sim_p y)$ 
for $x,y \in \CC^D$ 
if 
$p(x,w)=p(y,w)$ for all $w \in \CC^{D_w}$.
Note that under this definition, there may be $x,y \in \RR^D$
such that 
$(x \sim_G y)$ while
$(x \not \sim_p y)$ (due to some non-real valued, separating $w$). What is true is that  
$(x \not \sim_G y)$ implies
$(x \not \sim_p y)$.

 Define the polynomial over $\CC^{D + D + D_w}$:
	$q(x,y,w)= p(x,w)-p(y,w)$. This has degree $r$.
	For $i=1,..,m$, we can define the polynomials 
	$q_i(x,y,w_1,...,w_m):=q(x,y,w_i)$. 
	Together, these $q_i$ define a
	variety
	$V$
	in $\CC^{D + D + mD_w}$, which contains the bundle
	$$	\B=\{(x,y,w_1,\ldots,w_m)\in \CC^D \times \CC^D \times \CC^{D_w\times m}| \quad x \not \sim_p y \text{ but } p(x,w_i)=p(y,w_i), i=1,\ldots,m   \} $$

    By assumption, the set of $(x \sim_p y)$ must satisfy $q(x,y,w)=0$ for all $w$. 
    Taking the intersection over a sufficiently large finite collection of  such $w$ must stabilize and define a 
    strict subvariety $U$ of p-equivalent pairs. 
	$\B$ is obtained from $V$ by removing $U$. This can remove some of the components of $V$ as well as some nowhere-dense subsets of other components.
	Thus the Zariski closure $\bar{\B}$ must consist of some subset of the
	components of $V$. Note that 
	the Zariski closure does not increase the dimension of $\B$, which is bounded from above by
	$2D+mD_W-m$ (using the argument from Theorem~\ref{thm:useful} for the complex setting).
	Let us stratify $\bar{\B}$ into pure dimensional algebraic sets 
	$V_i$. 
	From Bezout's theorem (see~\cite[Chapter 18]{harris} and see~\cite{taoblog}, especially remark 2), 
	and the fact that $V$ is defined using the intersection  of $m$ varieties, each of 
	degree $r$, 
	each of these $V_i$ (made up of components of $V$) is of degree at most
	$r^m$ . There are at most $2D+mD_w-m$ such $i$.
	
	Next we project each of these $V_i$ onto $\CC^{mD_w}$. 
	By our assumption that $m\ge 2D+1$, we know that this projection is of dimension less than $mD_w$.
	The image of 
	a fixed $V_i$ can be stratified into constructible sets $V_{ij}$ of pure dimension. There are at most $mD_w-1$ of these. 
	From Bezout, the closure of
	each such 
	$V_{ij}$ is a variety of degree at most $r^m$ (and of this same pure dimension). Each $V_{ij}$ is contained in 
	an algebraic hypersurface of at most the same degree.
	Taking the union over these $(mD_w-1)(2D+mD_w-m)$
	hypersurfaces
	shows that the image of this projection must satisfy a single non-trivial polynomial equation $F$, where that polynomial's  
	degree is at most $(mD_w-1)(2D+mD_w-m)(r^m)$.

        Let us define a set of $w_1,...,w_m$, each in $\CC^{D_w}$ to be p-bad if there is an $x,y$, each in $\CC^D$, with
        $(x \not \sim_p y)$ but such that, for all $i$, we have $p(x,w_i)=p(y,w_i)$. By construction, the p-bad set lies in the
        zero set of $F.$
        Let us define a set of $w_1,...,w_m$, each in $\RR^{D_w}$ to be g-bad if there is an $x,y$, each in $\RR^D$, with
        $(x \not \sim_G y)$ but such that, for all $i$, we have $p(x,w_i)=p(y,w_i)$. 
        The G-bad set lies in the real locus of the p-bad set. Thus it lies in the zero sets of $F_r$ and $F_i$, the real polynomials
        defined by taking respectively the real/imaginary components of the coefficients of $F$. At least one of these two polynomials is non-zero.
        Such a non-zero polynomial gives us our $f$ in the statment of the lemma.
\end{proof}

It is not as clear how to cover the full real semi-algebraic setting.

	\section{A Toy Experiment}\label{sec:experiments}

To visualize the possible implications of our results to invariant machine learning, we consider the following toy example: we create random high dimensional point clouds in  $\RR^{3\times 1024}$ which reside in an $S_n, n=1024$ invariant `manifold' $\M$ of low-intrinsic dimension  $\Dint $. In fact, $\M$ is a union of two invariant `manifolds' $\M=\M_0\cup \M_1 $ of dimension $\Dint$, and we consider the problem of learning the resulting binary classification task. 

The binary classification task is visualized in Figure~\ref{fig:sort}(a). In this figure $\Dint=1$, each $\M_0,\M_1$ is a line in $\RR^{3\times 1024}$ and all its possible permutations,  and points in $\M_0,\M_1$ are projected onto $\RR^3$ for visualization. While this data may appear hopelessly entangled, using the permutation invariant mapping we describe in \eqref{eq:balan} with $m=2D+1=3 $ to embed $\M_0,\M_1$ into $\RR^3$ we obtain very good separation of the initial curves as shown in  Subplot (b). Note that the non-intersection of the images of $\M_0,\M_1$ is guaranteed by Proposition~\ref{prop:Sn}. 

In Figure~\ref{fig:sort}(c) we show the results obtained for the binary classification task by first computing the invariant embedding in \eqref{eq:balan} with randomly chosen weights, and then applying an MLP (Multi-Layer-Perceptron) to the resulting embedding. The results on train and test data are shown for various choices of intrinsic dimension $D=\Dint$ and embedding dimension $m$. In particular, for the $D=1,m=3$ case visualized in Figure~\ref{fig:sort}(a)-(b) we get $98 \% $ accuracy on the test dataset.

The diagonal entries in the tables show the accuracy obtained for varying intrinsic dimensions $D$ and embedding dimension  $m=2D+1 $. Recall that Proposition~\ref{prop:Sn} our embedding in separating for these $D,m$ values, and thus theoretically perfect separation can be obtained by applying an MLP to the embedding. The diagonal entries in the tables show that indeed high accuracy can be obtained for these $(D,m)$ pairs. At the same time, we also see that taking higher dimensional embeddings  $m>2D+1 $ leads to improved accuracy. This is parsimonious with the common observation that deep learning algorithms are more successful in the over-parameterized regime, as well as results on phase retrieval \cite{candes2013phaselift} and random linear projections \cite{baraniuk2009random}, where the embedding dimension needed for stable recovery is typically larger than the minimal dimension needed for injective embedding. In any case, we note that in all cases we obtain high accuracy with embedding dimension much smaller than the extrinsic dimension $3\times 1024=3072 $. 

Additional details on the experimental setup can be found in Appendix~\ref{app:exp}. Code for reconstructing our experiment can be found in \cite{code}.

\section{Conclusion and future work}
The main result of this paper is providing a small number of efficiently computable separating invariants for various group actions on point clouds. Many interesting questions remain.  One example is studying the optimal cardinality necessary for separation. As mentioned above, in phase retrieval it is known that the number of invariants needed for separation is slightly less than twice the dimension,  and we believe this is the case for other invariant separation problems we discuss here as well. Another important question, which is discussed e.g., in \cite{balan2022permutation,cahill2022group} is understanding how stable given separating invariants are: separating invariants are essentially an injective mapping from the quotient space into $\RR^m$. Stability in this context means that the natural metric on the quotient space should not be severely distorted by the injective mapping. 

Perhaps the most important challenge is translating the theoretical insights presented here into invariant learning algorithms with strong empirical performance, provable separation and universality properties, and reasonable computational complexity. 

A useful direction for reducing computational complexity is `settling' for generic separation which as we saw can often be achieved with a small computational burden. In general, the downside of this is that there is a low-dimensional singular set on which there is no separation. This disadvantage will only be significant, for a given learning task, if a significant percentage of the data resides on or near the singular set. Therefore it could be useful to understand what the singular sets of various generic separators are, and what the likelihood of encountering them in specific data is. 

We hope to address these questions in future work, and hope others are inspired to do so as well. 

\textbf{Acknowledgements} N.D. was supported by ISF grant 272/23 and a Horev fellowship.

\newpage
	\bibliographystyle{abbrv}
	\bibliography{bib_eff}
\newpage
\appendix	
\section{Appendix: Additional proofs}
In this appendix we give the proofs omitted in the main text.
\begin{proof}[Proof of Proposition~\ref{prop:uniOrbit}]
	Denote by $\pi$ the quotient map $\pi:\M\to \M/G$. We know (\cite{Munkres} page 140) that there exist unique continuous functions $\hat{f}^{inv}:\M/G\to \RR^N$ and $\hat{f}$ such that $\hat{f}^{inv}\circ \pi=\finv $ and $\hat f \circ \pi=f $. Since $\finv$ separates orbits it follows that $\hat{f}^{inv}$ is one-to-one on $\M/G$, and since $\pi$ is continuous $\pi(K)$ is compact. Thus $\hat{f}^{inv} $ is a continuous bijection of the compact set $\pi(K)$ onto its image and thus is a homeomorphism. We can now write
	$$f(x)=\hat f(\pi(x))=\left[\hat f\circ (\hat{f}^{inv})^{-1}\right] (\hat{f}^{inv}(\pi(x)))=\left[\hat f\circ (\hat{f}^{inv})^{-1}\right]\left(\finv(x)\right), \forall x\in K $$
	which gives us our result by setting $\fgeneral:\RR^N\to \RR$ to be some continuous extension of $\hat f\circ (\hat{f}^{inv})^{-1}$ from its compact domain $f^{inv}(K)\subseteq \RR^N $ to all of $\RR^N$ (such an extension exists by Tietze's extension theorem).
\end{proof}

\begin{proof}[Proof of Lemma~\ref{lem:isometries}]
	Let $X,Y $ be two point clouds with the same Gram matrix. We need to show that $X$ and $Y$ are related by a $Q$-isometry.
	
	Let us first deal with the case that $X$ has rank $d$. There exist a subset $I\subseteq \{1,2,\ldots,n\}$ with  $|I|=d$ such that the columns $x_i,i\in I$  span $\RR^d$. We can define a linear mapping on the column space of $X$ by specifying its values on  the basis elements $x_i, i\in I$ to be  
	$$Ux_i=y_i, i\in I.$$
	Let us denote by $\bar X$ and $\bar Y$ the $d\times d$ matrices obtained from $X$ and $Y$ by choosing the columns  $x_i, i\in I$ and  $y_i, i\in I$ respectively. Then $\bar X$ has rank $d$ and since $\bar{X}^TQ\bar X=\bar{Y}^TQ\bar{Y} $ we see that $\bar Y$ has rank $d$ as well. It is clear that $U$ is a $Q$-isometry since it preserve the bi-linear form on basis elements and hence on all of $\RR^d$. It remains to show that $Ux_i=y_i$ also for $i\not \in I$. Fix some $i\not \in I$. Since the columns of $\bar X$ and $\bar Y$ form a basis, there exist $\alpha,\beta \in \RR^d$ such that 
	$$x_i=\bar X \alpha, \quad y_i=\bar Y \beta .$$
	Since $\bar Y=U\bar X$ it is sufficient to show that $\alpha=\beta$, which follows from the fact that 
	$$\bar{X}^TQ\bar{X}\alpha=\bar{X}^TQx_i=\bar{Y}^TQy_i=\bar{Y}^TQ\bar{Y}\beta=\bar{X}^TQ\bar{X}\beta $$
	and as $\bar{X}^TQ\bar{X}$ has rank $d$ we obtain $\alpha=\beta$.
	
	We now assume that $Q$ is positive definite. Note that for every $m\in \NN$ and every  matrix $A\in \RR^{n\times m}$, the kernel of $A$ and $A^TQA$ are identical, since for any $v\in \RR^m$, the positive definiteness of $Q$ implies that $v^TA^TQAv$ is non-negative, and will be zero if and only if $Av=0$, if and only if $QAv=0$. In particular $A$ and $A^TQA$ will have the same rank.
	
	Let $r\leq d$ denote the rank of $X$, and let $I\subseteq \{1,2,\ldots,n\}$ with  $|I|=r$ such that the columns $x_i,i\in I$ span the column space of $X$, and so the $d\times r$ matrix $\bar X$ formed by choosing the columns $x_i,i\in I$ has rank $r$.  Let $\bar Y$ denoted the $d\times r$ matrix formed by choosing the columns $y_i,i\in I$. From our previous remark we have 
	$$r=\rank(X)=\rank(X^TQX)=\rank(Y^TQY)=\rank(Y) $$
	and similarly $\rank(\bar Y)=r$. We can define $U$ on the column space of $X$ by specifying its values on the basis elements
	$$Ux_i=y_i, i\in I. $$
	We see that $U$ is maps the $r$-dimensional column space of $X$ $Q$-isometrically onto the $r$-dimensional column space of $Y$. We can extend $U$ to a $Q$-isometry of $\RR^d$ by appropriately defining a $Q$-isometry which maps the $d-r$ dimensional space orthogonal to the column space of $X$ to the $d-r$ dimensional space orthogonal to the column space of $Y$. We then show that $Ux_i=y_i$ when $i\not \in I$ using the same argument used in the  first part of the lemma.
	
\end{proof}	

We now prove a characterization of the $SO(d)$ linear equivariant maps, as discussed in Remark~\ref{remark:phase}. 
\begin{proposition}\label{prop:equiv}
	Let $L:\RR^{d\times n}\to \RR^d $ be a linear $SO(d)$ equivariant mapping for $d> 2$, then there exists $w\in \RR^n$ such that 
	$$L(X)=Xw $$	
\end{proposition}
\begin{proof} 
We can write 
$$L(X)=\sum_{i=1}^n L(x_i) $$
where each $L_i:\RR^d \to \RR^d $ is linear. By considering $X$ for which only the $i$-th column is not zero we see that each $L_i$ must be $SO(d)$ equivariant as well. It is thus sufficient to show that $L_i$ must be of the form $L_i(x)=\lambda x $ for some  $\lambda  \in \RR$. We note that $L_i$ commutes with every $R\in SO(d) $, and so if $v\in \RR^d$ is an eigenvector of $R$ with eigenvalue $\lambda$, then $L_iv$ is an eigenvector of $R$ with the same eigenvalue.

When $d$ is odd, note that for every $v\in \RR^d$ there exists $R\in SO(d)$ such that $Rv=v$ but $R=-I_{d-1}$ on the orthogonal complement of $v$. It follows that $L_iv=\lambda v$ for some $\lambda \in \RR$. Since this is true for every $v$ we see that all vectors are eigenvectors of $L_i$ and so $L_i=\lambda I_d$. 

When $d$ is even and $d>2$, we can, given some $v_0\in \RR^d$, choose two additional vectors $v_1,v_2$ such that the three vectors are normal to each other. We can choose a matrix $R_{1}\in SO(d)$ which is equal to $I_2$ on the span of $v_0$ and $v_1$ and equal to $-I_{d-2}$ on the orthogonal complement. It follows that $L_iv_0$ is in the span of $v_0$ and $v_1$. We can also construct $R_2\in SO(d)$ which is equal to $I_2$ on the span of $v_0,v_2$ and is equal to $-I_{d-2}$ on the orthogonal l complement. It follows that $L_iv_0$ is in the span of $v_0$ and $v_2$, as well as in the span of $v_0$ and $v_1$. Thus $L_iv_0=\lambda v_0$ for some $\lambda \in \RR$. Since this holds for all vectors $v_0\in \RR^d$ we see that $L_i=\lambda I_d$.  	
	\end{proof}

We now prove that the invariants from \eqref{eq:det2} separate $SL(d)$-orbits of full rank matrices:
	\begin{proposition}\label{prop:special}
		For every $d,n$ with $n \geq d$, the polynomial invariants from equation~\ref{eq:det2} separates $SL(d)$-orbits on the set $\RR_{full}^{d \times n}$.
	\end{proposition}
	\begin{proof}
		Let $X,Y\in \RR_{full}^{d \times n} $ and assume that $X$ and $Y$ are not related by a special linear transformation. Since $X$ is full rank, it has $d$ columns which are linearly independent, for simplicity of notation we assume these are the first $d$ columns. It the first $d$ columns of $Y$ are not linearly independent, or more generally if 
		$$\det[x_1,\ldots,x_d]\neq \det[y_1,\ldots,y_d] $$
		then the polynomial $[1,2,\ldots,d]$ separates $X$ and $Y$. If these two determinants are equal then let $A$ be the matrix satisfying $Ax_i=y_i, i=1,\ldots,d$, we see that $\det(A)=1$. Denote $\tilde X=AX$.  By assumption, $AX\neq Y$ so  there exists some index $j, d<j\leq n$ such that $Ax_j\neq  y_j$. Since $y_1,\ldots,y_d$ span $\RR^d$, there exist $\alpha_1,\ldots,\alpha_d$ and $\beta_1,\ldots,\beta_d$ such that 
		$$Ax_j=\sum_{i=1}^d \alpha_iy_i, y_j=\sum_{i=1}^d \beta_iy_i  $$
		and since $Ax_j\neq y_j$ there exists some $k, 1\leq k \leq d$ such that $\alpha_k\neq \beta_k$. It follows that $X$ and $Y$ are separated by the polynomial $[1,2,\ldots,k-1,j,k+1,\ldots,d]$ since
		\begin{align*}
			\det[ x_1,\ldots, x_{k-1}, x_j, x_{k+1},\ldots, x_d]
			&=\det[ Ax_1,\ldots, Ax_{k-1}, Ax_j, Ax_{k+1},\ldots, Ax_d]\\ &=\det[y_1,\ldots,y_{k-1},\alpha_ky_k,y_{k+1},\ldots,y_d]\\
			&\neq \det[y_1,\ldots,y_{k-1},\beta_ky_k,y_{k+1},\ldots,y_d]\\
			&=    \det[y_1,\ldots,y_{k-1},y_j,y_{k+1},\ldots,y_d]
		\end{align*}
	\end{proof}

\section{Appendix: Experimental Setup}\label{app:exp}
The data for the binary classification experiment described i{app:exp}n Figure~\ref{fig:sort}(c) was generated as follows: we generated $D$ dimensional convex sets $\M_0,\M_1$ by randomly choosing $D+1$ points in $\RR^{3\times 1024} $, and considering all possible permutations of the columns of these point clouds by $S_n, n=1024 $. We generated $1,000$ training samples and $1,000$ test samples by randomly choosing between $\M_i, i=0,1$, then randomly choosing a probability vector $t\in \RR^{D+1} $ and using it do define a point cloud as a convex combination of the $D+1$ point clouds used to generate $\M_i$, and finally applying a random permutation. 

For classification we used the same MLP architecture used in Point-Net \cite{qi2017pointnet}, with input dimension $m$ varying as shown in the table, and three hidden layers of dimensions $1,024,512,256 $. Batch normalization was applied to each layer, we used dropout with $p=0.3$ and used Adam for optimization with learning rate $0.001$. Each entry in the tables in Figure~\ref{fig:sort}(c) shows the average accuracy of ten random initializations of the embedding and network parameters.   

For visualization purposes, in Figure~\ref{fig:sort}(a) we only show $30 $ permutations applied to the original lines and not all possible permutations. In Figure~\ref{fig:sort}(a)-(b) we also took $n=10$ rather than $n=1024$ as used in  Figure~\ref{fig:sort}(c), as for larger values of $n$ the lines in (b) have more oscillations and are more difficult to visualize.
	
\end{document}

%% file: group_table.tex
\begin{center}
	\begin{tabular}{c c c c c c } 
		\hline
		Group   & Domain  & separators &complexity&num separators  & num generators                      \\ 
		\hline
		\hline
		$S_n$  &  $\RR^{d\times n}$   & $\langle w^{(2)},sort(X^Tw^{(1)})\rangle$ & $\bigO(nd+n log(n))$ & 2nd+1 & $\bigO\left({n+d \choose d}\right)$\\
		\hline
		$O(d)$ &  $\RR^{d\times n}$   &   $\|Xw\|^2$           & $\bigO(n\cdot d)$       & 2nd+1 & $\bigO(n^2) $ \\ 
		\hline
		$SO(d)$ &  $\RR^{d\times n}$  &   $\|Xw\|^2+\det(XW)$ & $\bigO(n\cdot d^2) $   & 2nd+1 &  $\bigO\left( {n \choose d} \right)$\\ 
		\hline
		$O(d-1,1)$ & $\RR^{d\times n}_{full}$ & $\langle Xw,QXw\rangle$ & $\bigO(n\cdot d)$ &2nd+1 & $\bigO(n^2) $  \\
		\hline
		$SL(d) $ & $\RR^{d\times n}_{full}$ & $det(XW)$ & $\bigO(nd^2) $ & 2nd+1 & $\bigO \left({n \choose d}\right)$ 
		\\
		\hline
		$GL(d)$ & $\RR^{d\times n}_{full}$ & $det^2(XW)\cdot det^{-1}(XX^T) $ & $\bigO(nd^2)$ & 2nd+1 & $1$ (only constants)
	\end{tabular}
\end{center}